\documentclass[lettersize,journal]{IEEEtran}
\usepackage{amsmath,amsfonts,bm}
\usepackage{algorithm}
\usepackage{array}
\usepackage{subcaption}
\usepackage{textcomp}
\usepackage{stfloats}
\usepackage{url}
\usepackage{verbatim}
\usepackage{graphicx}
\usepackage{cite}
\usepackage[capitalise]{cleveref}
\crefname{assumption}{Assumption}{Assumptions}
\crefname{problem}{Problem}{Problems}
\usepackage{selinput}\SelectInputMappings{adieresis={ä},germandbls={ß}}
\usepackage{epsfig}
\usepackage{times}
\usepackage{float}
\usepackage{tikz,pgfplots}
\usepackage[english]{babel}
\usepackage{algpseudocode}
\usepackage{enumitem}
\usepackage{verbatim}
\usepackage{makecell}
\usepackage{multirow}
\usepackage{xcolor}
\usepackage{soul}
\usepackage[switch]{lineno} 
\usetikzlibrary{arrows,shapes,backgrounds,patterns,fadings,matrix,arrows,calc,
	intersections,decorations.markings,
	positioning,arrows.meta}
\usepgfplotslibrary{fillbetween}
\usepgfplotslibrary{statistics}
\pgfplotsset{width=5\columnwidth /5, compat = 1.13,
	height = 60\columnwidth /100, grid= major,
	legend cell align = left, ticklabel style = {font=\scriptsize},
	every axis label/.append style={font=\small},
	legend style = {font={\scriptsize}},title style={yshift=-7pt, font = \small} }

\newtheorem{assumption}{\bf{Assumption}}
\newtheorem{theorem}{\bf{Theorem}}

\newtheorem{lemma}{\bf{Lemma}}
\newtheorem{remark}{\bf{Remark}}
\newtheorem{definition}{\bf{Definition}}
\newtheorem{proposition}{\bf{Proposition}}

\crefname@preamble{assumption}{Assumption}{Assumptions}
\crefname@preamble{lemma}{Lemma}{Lemmas}
\crefname@preamble{property}{Property}{Properties}

\xdefinecolor{redZ}{RGB}{234,29,93}
\xdefinecolor{blueZ}{RGB}{19,106,213}
\xdefinecolor{yellowZ}{RGB}{251,138,46}

\xdefinecolor{blue_revise}{HTML}{0000FF}

\xdefinecolor{red}{HTML}{d62728}
\xdefinecolor{green}{HTML}{2ca02c}
\xdefinecolor{blue}{HTML}{1f77be}
\xdefinecolor{cyan}{HTML}{ff7f0e}
\xdefinecolor{magenta}{HTML}{9467db}
\xdefinecolor{yellow}{HTML}{8c564b}
\xdefinecolor{gray}{HTML}{7f7f7f}

\DeclareMathOperator*{\argmax}{arg\,max}

\newif\ifProof
\Prooftrue
\newif\ifarxiv
\arxivtrue
\newif\ifFastDraft
\FastDraftfalse

\makeatletter
\def\@IEEEaftertitletext{\vspace{-0pt}}
\makeatother

\begin{document}
%	\linenumbers
	\title{
		% Non-Reciprocal Cooperative Online Learning for Multi-Robot Consensus under Computational Delay with Distributed Gaussian Processes
		% Cooperative Learning for Multi-Robot Control under Time-varying Heterogeneous Computational Delay
		% Asynchronous Cooperative Online Learning for Multi-Robot Control under Heterogeneous Computational Delay
		% Adjoint-Driven Multi-Robot Control with Non-Reciprocal Gaussian Process Learning
		% Asynchronous Cooperative Online Learning for Multi-Robot Control with Gaussian Processes under Heterogeneous Computational Delays
		Asynchronous Cooperative Online Learning for Multi-Robot Control under Computational Delays
	}
	
	\author{
		Xiaobing Dai, Zewen Yang~\IEEEmembership{Member, IEEE}, Wei Ren,~\IEEEmembership{Fellow, IEEE},  Sandra Hirche,~\IEEEmembership{Fellow, IEEE}
		\thanks{
			This work was supported by the Federal Ministry of Education and Research of Germany in the programme of ``Souverän. Digital. Vernetzt.'' under joint project 6G-life with project identification number: 16KISK002. \textit{(Corresponding author: Zewen Yang.)}
		}
		\thanks{
			Xiaobing Dai and Sandra Hirche are with the Chair of Information-oriented Control, Technical University of Munich, 80333 Munich, Germany (email: xiaobing.dai, hirche@tum.de).
		}
		\thanks{
			Zewen Yang is with the Chair of Robotics and Systems Intelligence, Technical University of Munich, 80992 Munich, Germany (email: zewen.yang@tum.de).
		}
		\thanks{
			Wei Ren is with the Department of Electrical and Computer Engineering, University of California at Riverside, Riverside, CA 92521 USA (e-mail: ren@ece.ucr.edu).
		}
	}

	\maketitle
	
	\begin{abstract}
		Ensuring the safe operation of multi-agent systems (MASs) under uncertain environments is crucial for cooperative robotic, where external disturbances and inaccurate dynamic models can significantly compromise performance and reliability.
		To address this challenge, calibrated machine learning models, particularly Gaussian process (GP) regression, are extensively employed due to their interpretable performance quantification.
		As the interconnected communication of MASs facilitates cooperative learning, agents are able to enhance learning performance by exchanging local GP inferences with their neighbors and aggregating the received information via distributed GP strategies.
		However, variations in computational power and prediction tasks among agents inevitably lead to heterogeneous computational delays and differences in query points, which are often overlooked in existing aggregation methods.
		To overcome these limitations, this work proposes an asynchronous cooperative learning strategy that explicitly accounts for prediction accuracy, query point variations and delay effects. 
		Additionally, a distributed control law based on an adjoint MAS is developed to ensure the desired control performance. 
		Simulations on unmanned surface vehicles validate the effectiveness of the proposed approach, demonstrating substantial improvements in both learning and control performance compared to the state-of-the-art approaches.
	\end{abstract}
	
	\section{Introduction}
	\label{section_introduction}	
	
	\subsection{Background}
	
	Safety-critical multi-agent system (MAS) control has drawn large attention in robotics including applications in unmanned vehicles \cite{yang2025safe} and manipulators \cite{zhang2020modular}, where individual robots cooperate with each other to accomplish tasks such as consensus \cite{panagou2015distributed}, formation \cite{wang2017distributed}, coverage \cite{cortes2004coverage, zhu2013distributed} and optimization \cite{zhu2011distributed}. 
	However, the exact agent dynamics is often hard to obtain because of high model complexity and environmental uncertainties, imposing challenges for designing model-based controllers \cite{lederer2022cooperative}. 
	To this end, machine learning (ML) techniques are employed on each agent to predict unknown components in its dynamics based on collected data. 
	
	\subsection{Related Works}
	
	Notably, the dependency of high prediction accuracy and better control performance for MASs is demonstrated in many existing studies \cite{dai2024decentralized}.
	Consequently, cooperative learning is a promising methodology that leverages inter-agent collaboration to improve learning quality \cite{yang2024cooperative}.
	Specifically, cooperative neural network (NN)-based MAS control is extensively studied \cite{zhu2023neural}, assuming that the unknown system is approximated as a linear combination of predefined nonlinear features. 
	However, the accuracy of NN relies on the choice of nonlinear features \cite{zhu2023neural}.
	Practical quantification of prediction error remains an open challenge, where only overly conservative results exist \cite{lamperski2024approximation}, limiting its applications in safety-critical scenarios. 
	
	Gaussian process (GP) regression-based cooperative learning is another popular approach due to its modeling flexibility for arbitrary continuous functions \cite{williams2006gaussian} and interpretable error quantification \cite{srinivas2012information, hashimoto2022learning}.
	Cooperative GPs require agents sharing their predictions with neighbors and aggregating received neighbor predictions \cite{yang2024cooperative}.
	However, these aggregation techniques demand that agents compute predictions for the neighboring agents, leading to significant local computational burden.
	Moreover, some resource-efficient distributed GPs are investigated recently for multi-agent systems \cite{ding2024resource, yuan2020communication, yuan2024lightweight}. 
	Similarly, decentralized GP frameworks such as nested GPs \cite{kontoudis2021decentralized} enable cooperative learning across multiple agents with limited communication.
	While these approaches focus on computational efficiency, scalability and communication constraints, they typically assume that predictions are synchronously available or do not explicitly account for the impact of computational delays.
	Notably, when computational delays are considered in existing cooperative GP approaches, the predictions received from neighbors, which are derived from their local data, arrive in a delayed manner.
	However, only few works take the limitation on local computational power into account.
	Although cooperative learning enhances predictive performance through information sharing, its accuracy depends on the quality of each agent’s individual GP model.
	To ensure desired learning performance during operation, the effectiveness of online cooperative learning is shown in the previous works \cite{dai2024cooperative}, allowing individual GP models to be updated with newly collected data. 
	However, the high computational intensity of online learning induces non-negligible computational delays.
	
	The learning-based MAS control with delays has drawn huge attention, but existing research focuses only on the communication delays \cite{zhang2022adaptive}. 
	Unlike well-understood heterogeneous communication delays, which allow continuous data transmission with different time shifts, computational delay directly blocks additional operations during execution \cite{dai2023can}. 
	This difference prevents the straightforward extension of existing results to the setting with computational delays from machine learning. 
	Consequently, the problem of computational delays in cooperative online learning-based MAS control remains unexplored, and only few preliminary works \cite{yang2025asynchronous} explored the effect of delayed learning in a single-agent system. 
	Specifically, the computational delay induced by GP regression is first analyzed in \cite{dai2023can}, revealing the coupled effect of prediction accuracy and computational delay on control performance.
	However, the extension to MAS with distributed cooperative GP is nontrivial due to different computational power and prediction tasks across agents.
	Therefore, it is urgent to investigate a cooperative online learning-based distributed control strategy for MAS under heterogeneous computational delays. 
	
	\subsection{Contributions}
	
	In this paper, the main contributions are the answers to the following three questions, which are induced by heterogeneous computational delay in cooperative online learning-based multi-agent system control.
	\begin{enumerate}
		\item[\textbf{Q1}] \emph{Can the delayed prediction from neighboring agents improve individual learning performance?}
		\item[\textbf{Q2}] \emph{How to properly utilize the asynchronous predictions from neighbors with different query points?}
		\item[\textbf{Q3}] \emph{How to design a distributed controller leveraging asynchronous cooperative online learning?}
		\item[\textbf{Q4}] \emph{Does the proposed controller have a formal guarantee on control performance in MASs?}
	\end{enumerate}
	Here, we focus on a MAS modeled in the Euler-Lagrange form with unknown disturbances motivated by robotics and mechatronic systems, whose control objective is to achieve a second-order consensus. 
	GP regression is deployed on each agent to predict unknown components using a data set, which is online updated with newly collected measurements. 
	Moreover, each GP model only computes and shares the predictions of its own state with neighbors, avoiding the inference of unknown dynamics based on neighbors' states as in \cite{yang2024cooperative}.
	To aggregate received predictions, an asynchronous distributed cooperative learning strategy is proposed for MASs, which accounts for prediction accuracy, query point difference, and delay effects.
	Furthermore, a distributed control law is devised based on an introduced adjoint MAS and the proposed cooperative online learning strategy, guaranteeing the achievement of practical second-order consensus. 
	The effectiveness of the cooperative learning and distributed control strategy is demonstrated through simulations on unmanned surface vehicles, highlighting improvements in both learning and control performance. 
	
	\subsection{Graph Theory}
	
	The communication network among $N \in \mathbb{N}_+$ agents is modeled as an undirected graph $\mathcal{G} = \{ \mathcal{V}, \mathcal{E} \}$, with the vertex set $\mathcal{V} = \{ 1, \cdots, N \}$ and the edge set $\mathcal{E} \subset \mathcal{V} \times \mathcal{V}$.
	Specifically, $(i, j) \in \mathcal{E}$ indicates the existence of a communication channel between agent $i$ and $j$ for any $i, j \in \mathcal{V}$.
	The communication graph is characterized by adjacent matrix $\bm{A} \in \mathbb{R}^{N \times N}$, whose entries $a_{i,j}$ at the $i$-th row and $j$-th column are set as $a_{i,j} = a_{j,i} = 1$ if $(i, j) \in \mathcal{E}$, and $a_{i,j} = a_{j,i} = 0$ otherwise. 
	The neighbor set of agent $i$ is defined as $\mathcal{N}_i = \{ j \in \mathcal{V} | a_{i,j} \ne 0 \} \subset \mathcal{V}$.
	Moreover, the input degree matrix is defined as $\bm{D} = \mathrm{diag}(d_1, \cdots, d_N)$ with $d_i = \sum_{j = 1}^N a_{i,j}$ for any $i \in \mathcal{V}$, so that the Laplacian matrix is written as $\bm{L} = \bm{D} - \bm{A} \in \mathbb{R}^{N \times N}$.
	%%%%%%%%%%%%%%%%%%%%%%%%%%%%%%%%%%%%%%%%%%%%%%%%%%%%%%%%%%%%%%%%%%%%
	%%%%%%%%%%%%%%%%%%%%%%%%%%%%%%%%%%%%%%%%%%%%%%%%%%%%%%%%%%%%%%%%%%%%
	\section{Problem Setting}
	\label{section_problem}
	%%%%%%%%%%%%%%%%%%%%%%%%%%%%%%%%%%%%%%%%%%%%%%%%%%%%%%%%%%%%%%%%%%%%
	%%%%%%%%%%%%%%%%%%%%%%%%%%%%%%%%%%%%%%%%%%%%%%%%%%%%%%%%%%%%%%%%%%%%
	
	%%%%%%%%%%%%%%%%%%%%%%%%%%%%%%%%%%%%%%%%%%%%%%%%%%%%%%%%%%%%%%%%%%%%
	\subsection{Problem Formulation}\label{subsection_problem}
	%%%%%%%%%%%%%%%%%%%%%%%%%%%%%%%%%%%%%%%%%%%%%%%%%%%%%%%%%%%%%%%%%%%%
	In this paper, a homogeneous MAS with $N \in \mathbb{N}_+$ agents is considered, where the dynamics of each agent $i = 1, \cdots, N$ in Euler-Lagrange form is written as
	\begin{align}
		\label{eqn_agent_EL_dynamics}
		\bm{M}(\bm{q}_i(t)) \ddot{\bm{q}}_i(t) + \bm{C}([\bm{q}_i^T(t), & \dot{\bm{q}}_i^T(t)]^T) \dot{\bm{q}}_i(t) + \bm{g}(\bm{q}_i(t)) \\
		&= \bm{\tau}_i(t) + \bm{d}([\bm{q}_i^T(t), \dot{\bm{q}}_i^T(t)]^T), \nonumber
	\end{align}
	where $\bm{q}_i \in \mathbb{X}_{q,0} \subset \mathbb{R}^n$ and $\dot{\bm{q}}_i \in \mathbb{X}_{q,1} \subset \mathbb{R}^n$ denote the generalized coordinate and velocity of robot $i$ forming the agent state $\bm{x}_i = [\bm{q}_i^T, \dot{\bm{q}}_i^T]^T \in \mathbb{X} \subset \mathbb{R}^{2 n}$ with dimension $n \in \mathbb{N}_+$.
	The known mass matrix $\bm{M}(\cdot): \mathbb{X}_{q,0} \to \mathbb{R}^{n \times n}$, Coriolis force matrix $\bm{C}(\cdot): \mathbb{X} \to \mathbb{R}^{n \times n}$ and gravity force vector $\bm{g}(\cdot): \mathbb{X}_{q,0} \to \mathbb{R}^n$ are obtained from nominal robot dynamics.
	The control input of the agent dynamics \eqref{eqn_agent_EL_dynamics} denotes $\bm{\tau}_i \in \mathbb{R}^n$.
	The unknown continuous $\bm{d}(\cdot): \mathbb{X} \to \mathbb{R}^n$ encodes all uncertainties, which is assumed identical to all agents. 
	Moreover, the system \eqref{eqn_agent_EL_dynamics} satisfies the following assumption.
	
	\begin{assumption} \label{assumption_dynamics}
		There exists strictly positive constant $\underline{M} \in \mathbb{R}_+$ such that $\underline{\sigma}(\bm{M}(\bm{q})) \ge 1 / \underline{M}$ for any $\bm{q} \in \mathbb{X}_q$.
	\end{assumption}
	
	\cref{assumption_dynamics} only requires the positive definite of the mass matrix $\bm{M}(\bm{q})$ for any $\bm{q} \in \mathbb{X}_q$, which is satisfied in most engineering applications including underwater vehicles \cite{yang2025safe} and manipulators \cite{zhang2020modular}.
	\cref{assumption_dynamics} guarantees that the inertia matrix $\bm{M}(\bm{q}_i)$ remains uniformly positive definite for all admissible configurations, which ensures the uniform nonsingularity and well-posedness of \eqref{eqn_agent_EL_dynamics}.
	Therefore, this assumption imposes no practical restriction.
	
	The agents communicate through a network modeled as an undirected graph $\mathcal{G}$, which satisfies the following assumption.
	
	\begin{assumption} \label{assumption_graph}
		The undirected graph $\mathcal{G}$ is connected.
	\end{assumption}
	
	\cref{assumption_graph} is common in many MAS control literature, e.g., \cite{dai2024decentralized}, which allows the information of each agent to be transmitted to other agents within finite steps.
	
	The control task is to achieve second-order approximate consensus, which is mathematically formulated as follows.
	\begin{definition} \label{definition_consensus}
		A second-order multi-agent system with \eqref{eqn_agent_EL_dynamics} achieves approximate consensus, if there exit well-defined positive constants $\epsilon_0, \epsilon_1 \in \mathbb{R}_{0,+}$ such that
		\begin{align}
			\lim\nolimits_{t \to \infty} \| \bm{q}_i(t) - \bm{q}_j(t) \| \le \epsilon_0, && \lim\nolimits_{t \to \infty} \| \dot{\bm{q}}_i(t) \| \le \epsilon_1
		\end{align}
		hold for any $i, j \in \mathcal{V}$ and $i \ne j$.
	\end{definition}
	
	Classical second-order consensus formulations typically require asymptotic agreement of both positions and velocities, allowing the common velocity to be nonzero \cite{hou2017consensus}.
	In contrast, \cref{definition_consensus} additionally enforces convergence of all agent velocities $\dot{\bm{q}}_i$ to zero.
	This formulation avoids persistent collective drift in robotic coordination tasks \cite{lin2010consensus} and achieves stationary consensus behavior for the considered system \eqref{eqn_agent_EL_dynamics}.
	Moreover, the approximate consensus in \cref{definition_consensus} is widely studied in MAS \cite{amelina2015approximate, jiang2023reachable}, where agents are required to reach agreement within a bounded tolerance.
	In this work, the approximate consensus is considered due to the uncertain dynamics \eqref{eqn_agent_EL_dynamics} and practically non-vanishing estimation error for unknown $\bm{d}(\cdot)$.
	
	\begin{remark}
		The generalized coordinates $\bm{q}_i$ in \eqref{eqn_agent_EL_dynamics} are not restricted to physical Cartesian positions and may also represent abstract coordination variables, such as manipulator joint configurations or task-space synchronization states. 
		Even when $\bm{q}_i$ are the Cartesian positions, collision-free coordination can be interpreted through offset consensus, where consensus of $\bm{q}_i$ yields constant relative distances between agents.
		Consequently, physical collisions are avoided while preserving the same stability analysis framework.
	\end{remark}
	
	\subsection{Data-driven Distributed Control}
	To achieve the second-order approximate consensus in \cref{definition_consensus}, a distributed control law is designed as
	\begin{align} \label{eqn_control_law}
		\bm{\tau}_i(t) = \bm{\tau}_{s,i}(\cup_{j \in \{ i, \mathcal{N}_i \}} \bm{x}_j) + \bm{\tau}_{m,i}(\bm{x}_i) + \bm{\tau}_{c,i}(t) - \hat{\bm{d}}_i(t)
	\end{align}
	for $\forall i \in \mathcal{V}$, where the consensus law $\bm{\tau}_{s,i}(\cdot)$ and the model related feedforward term $\bm{\tau}_{m,i}(\cdot)$ are written as 
	\begin{align}
		&\bm{\tau}_{s,i}(\cup_{j \in \{ i, \mathcal{N}_i \}} \bm{x}_j ) = - \bm{M}(\bm{q}_i) (c_0 \bm{e}_{q,i} + c_1 \dot{\bm{e}}_{q,i} ), \label{eqn_control_law_consensus} \\
		&\bm{\tau}_{m,i}(\bm{x}_i) = -  \bm{M}(\bm{q}_i) (c_3 \dot{\bm{q}}_i + c_2 \bm{q}_i) + \bm{C}(\bm{x}_i) \dot{\bm{q}}_i + \bm{g}(\bm{q}_i) \label{eqn_control_law_model}
	\end{align}
	with the consensus error $\bm{e}_{q,i} = \sum\nolimits_{j \in \mathcal{N}_i} (\bm{q}_i - \bm{q}_j)$.
	The correction term $\bm{\tau}_{c,i}(\cdot)$ from the proposed adjoint MAS is designed later in \cref{subsection_adjoint_MAS}, and the functions $\hat{\bm{d}}_i(\cdot): \mathbb{R}_{0,+} \to \mathbb{R}^n$ serve as compensation terms for $\bm{d}(\bm{x}_i(\cdot))$, $\forall i \in \mathcal{V}$.
	The control gains $c_i \in \mathbb{R}$ with $i = 0, \cdots, 3$ are designed to satisfy the conditions derived later in \cref{remark_condition_c}.
	
	\begin{remark}
		The proposed control framework \eqref{eqn_control_law} consists of $4$ components.
		The consensus term $\tau_{s,i}$ is inspired by second-order consensus protocols \cite{lin2010consensus, hou2017consensus}, and is based on position and velocity consensus errors.
		The model-based term $\tau_{m,i}$ follows the classical computed-torque control \cite{middletone1986adaptive}, which is widely employed to compensate for known system dynamics.
		The disturbance compensation $\hat{\bm{d}}_i(t)$ is obtained via machine learning as shown in \cite{dai2024cooperative}. 
		Notably, the proposed approach differs from these methods by explicitly accounting for asynchronous online learning with heterogeneous computational delays.
		Finally, the correction term $\tau_{c,i}$ is newly introduced in this work. 
		It is constructed based on an adjoint multi-agent system that serves as a reference, together with an additional component designed to mitigate the query mismatch from asynchronous learning.
		Further details are provided in \cref{section_nonReciprocal_learning}.
		\looseness=-1
	\end{remark}
	
	Specifically, $\hat{\bm{d}}_i(t)$ is estimated by agent $i \in \mathcal{V}$ using an online ML with time-varying data sets $\mathbb{D}_i(t)$, which denotes
	\begin{align} \label{eqn_data_set}
		\mathbb{D}_i(t) = \{ \bm{x}_i(t_{i,k}), \bm{y}_i(t_{i,k}) \}_{k = \underline{k}_i(t), \cdots, \bar{k}_i(t)},
	\end{align}
	leading to $|\mathbb{D}_i(t)| \le \bar{N} \in \mathbb{N}_+$, $\forall t \in \mathbb{R}_{0,+}$ with $\bar{N} \in \mathbb{N}_+$.
	The discrete time sequence $\{ t_{i,k} \}_{k \in \mathbb{N}}$ comprises the time instances at which the ML model on agent $i$ becomes available for next computation, which is defined as $t_{i,0} = 0$, $t_{i,k + 1} = t_{i,k} + \Delta_i(t_{i,k})$, $\forall k \in \mathbb{N}$, where $\Delta_i(t_{i,k})$ is the computation time for ML initiated at $t_{i,k}$ on agent $i$.
	The index function $\bar{k}_i(\cdot): \mathbb{R}_{0,+} \to \mathbb{N}$ in \eqref{eqn_data_set} returns the closest index of the time instances for new learning process on agent $i$, i.e., $\bar{k}_i(t) = \argmax\nolimits_{k \in \mathbb{N}} \{ t_{i,k} | t_{i,k} + \Delta_i(t_{i,k}) \le t, \forall k \in \mathbb{N} \}$ with $\underline{k}_i(t) = \max\{ 0, \bar{k}_i(t) - \bar{N} + 1 \}$.
	The vector $\bm{y}_i \in \mathbb{R}^n$ in \eqref{eqn_data_set} is the noisy observation of unknown $\bm{d}(\bm{x}_i)$ evaluated at $\bm{x}_i$, which satisfies the following assumption.
	
	\begin{assumption}
		\label{assumption_data_set}
		The data pair $\{ \bm{x}_i(t_{i,k}), \bm{y}_i(t_{i,k}) \}$ is only available on agent $i \in \mathcal{V}$ at time instance $t_{i,k}$ with $k \in \mathbb{N}$, where $\bm{y}_i(t_{i,k}) = \bm{d}(\bm{x}_i(t_{i,k})) + \bm{w}_i(t_{i,k})$ with i.i.d measurement noise $\bm{w}_i(\cdot) = [w_{i,1}(\cdot), \cdots, w_{i,n}(\cdot)]^T$.
		Moreover, it holds $| w_{i,p}(t_{i,k}) | \le \bar{w}_p$, $\forall p = 1, \cdots, n$ with known $\bar{w}_p \in \mathbb{R}_+$.
	\end{assumption}
	
	\cref{assumption_data_set} ensures the local availability of data without data sharing between agents. 
	Moreover, \cref{assumption_data_set} permits noisy measurements $\bm{y}_i(\cdot)$, which can be practically obtained through numerical methods.
	
	Motivated by \textbf{Q1}–\textbf{Q4}, the control objective of this paper is formally stated as follows.
	
	\textbf{Problem:}
	Consider the uncertain multi-agent system \eqref{eqn_agent_EL_dynamics} interconnected via an undirected communication graph $\mathcal{G}$ satisfying \cref{assumption_dynamics,assumption_graph}.
	The goal is to design an asynchronous cooperative online learning strategy with heterogeneous computational delays $\Delta_i(\cdot)$ and to devise a distributed learning-based control law, such that practical second-order consensus in \cref{definition_consensus} is achieved.
	 
	To address this problem, a distributed control framework is proposed with asynchronous cooperative learning as in \cref{figure_framework}.
	
	\begin{figure}[t]
		\centering
		\includegraphics[width=0.48\textwidth]{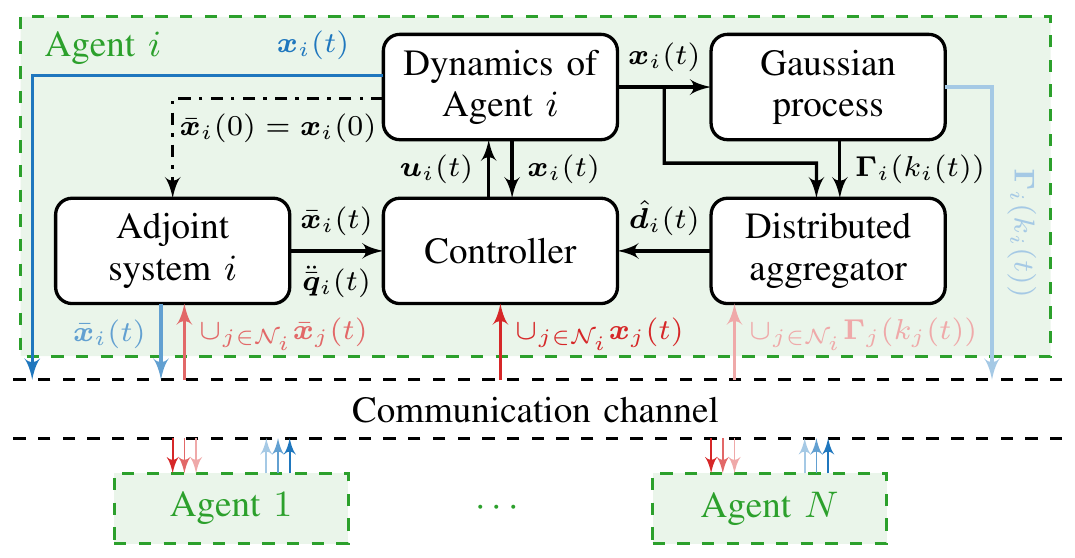}
		\caption{ 
			Block diagram of the proposed distributed control framework with asynchronous cooperative online learning, where the {\color{blue} blue} lines indicate the signal to be transmitted to neighbors and the {\color{red} red} lines shows the information received from neighbors.
			The dashed dot line infers the initialization step implemented only once at time $t=0$.
		}
		\label{figure_framework}
	\end{figure}
	
	%%%%%%%%%%%%%%%%%%%%%%%%%%%%%%%%%%%%%%%%%%%%%%%%%%%%%%%%%%%%%%%%%%%
	\section{Asynchronous Cooperative Online Learning}
	\label{section_nonReciprocal_learning}
	%%%%%%%%%%%%%%%%%%%%%%%%%%%%%%%%%%%%%%%%%%%%%%%%%%%%%%%%%%%%%%%%%%%
	In this section, a cooperative online learning strategy is proposed to answer \textbf{Q1} and \textbf{Q2}, which concern the feasibility and methodology of utilizing delayed neighboring predictions.
	
	%%%%%%%%%%%%%%%%%%%%%%%%%%%%%%%%%%%%%%%%%%%%%%%%%%%%%%%%%%%%%%%%%%%
	\subsection{Gaussian Process Regression}
	\label{subsection_GP}
	%%%%%%%%%%%%%%%%%%%%%%%%%%%%%%%%%%%%%%%%%%%%%%%%%%%%%%%%%%%%%%%%%%%%
	To infer the unknown function $\bm{d}(\cdot) = [d_1(\cdot), \cdots, d_n(\cdot)]^T$ in \eqref{eqn_agent_EL_dynamics}, GP regression is employed, which assumes $d_p(\cdot)$ with $p = 1, \cdots, n$ belongs to the reproduced kernel Hilbert space (RKHS) corresponding to kernel functions $\kappa_p(\cdot, \cdot): \mathbb{X} \times \mathbb{X} \to \mathbb{R}_{0,+}$.
	Moreover, each unknown function $d_p(\cdot)$ and kernel $\kappa_p(\cdot, \cdot)$ with $p = 1, \cdots, n$ satisfy the following assumption.
	\begin{assumption} \label{assumption_GP}
		The kernel function $\kappa_p(\cdot, \cdot)$ is stationary and Lipschitz continuous, which indicates the existence of a well-defined constant $L_{\kappa,p} \in \mathbb{R}_{0,+}$ such that $| \kappa_p(\bm{x}_1, \bm{x}) - \kappa_p(\bm{x}_2, \bm{x}) | \le L_{\kappa,p} \| \bm{x}_1 - \bm{x}_2 \|$ for any $\bm{x}_1, \bm{x}_2, \bm{x} \in \mathbb{X}$.
		Moreover, the unknown function $d_p(\cdot)$ has the bounded RKHS norm with well-defined positive constant $\bar{d}_p \in \mathbb{R}_{0,+}$, i.e., $\| d_p \|_{\kappa_p} = \langle d_p(\cdot), d_p(\cdot) \rangle_{\kappa_p}^{1/2} \le \bar{d}_p$ for any $p = 1, \cdots, n$. 
	\end{assumption}
	
	\cref{assumption_GP} is commonly found in most universal kernels such as squared exponential kernel and Mat$\acute{\mathrm{e}}$rn kernels.
	
	\begin{remark}
		The bounded RKHS norm in \cref{assumption_GP} is a standard modeling assumption in GP when deriving prediction error bounds \cite{srinivas2012information, hashimoto2022learning}, which imposes the smoothness on the unknown $d_p(\cdot)$. 
		In practice, the exact RKHS norm bound $\bar{d}_p$ is typically unknown, but it can be conservatively approximated using data-driven methods \cite{hashimoto2022learning, tokmak2024pacsbo, tokmak2025safe}. 
		While a tighter estimate of $\bar{d}_p$ leads to less conservative prediction error bounds, the theoretical analysis in this paper remains valid with any conservative upper bound. 
		Investigating a tighter bound $\bar{d}_p$ is an interesting direction for future work.
	\end{remark}

	Utilizing $\mathbb{D}_i(t)$ in \eqref{eqn_data_set} with $\bm{y}_i(\cdot) = [y_{i,1}(\cdot), \cdots, y_{i,n}(\cdot)]^T$ and measurement noise bound $\bar{w}_p$ for each dimension $p$, the posterior mean $\mu_p(\bm{x}, \mathbb{D}_i(t))$ and variance $\sigma_p^2(\bm{x}, \mathbb{D}_i(t))$ for $d_p(\cdot)$ with $p = 1, \cdots, n$ at any $\bm{x} \in \mathbb{X}$ denote
	\begin{align} \label{eqn_GP_prediction}
		&\mu_p(\bm{x},\! \mathbb{D}_i(t)) \!=\! \bm{\kappa}_p^T(\bm{x},\! \mathbb{D}_i(t)) \bar{\bm{K}}_p^{-1}(\mathbb{D}_i(t)) \bm{\zeta}_p(\mathbb{D}_i(t)), \\
		&\sigma_p^2(\bm{x},\! \mathbb{D}_i(t)) \!=\! \sigma_{s,p}^2 \!-\! \bm{\kappa}_p^T(\bm{x},\! \mathbb{D}_i(t)) \bar{\bm{K}}_p^{-1}(\mathbb{D}_i(t)) \bm{\kappa}_p(\bm{x},\! \mathbb{D}_i(t)), \nonumber
	\end{align}
	where $\sigma_{s,p}^2 = {\kappa_p(\bm{x}, \bm{x})}$ is constant for any $\bm{x} \in \mathbb{X}$ due to the stationary kernel. 
	The kernel vector is $\bm{\kappa}_p^T(\bm{x}, \mathbb{D}_i(t)) = [\kappa_p(\bm{x}, \bm{x}_i(t_{i, \underline{k}_i(t)})), \cdots, \kappa_p(\bm{x}, \bm{x}_i(t_{i, \bar{k}_i(t)}))]^T$ and the kernel gram matrix is $\bar{\bm{K}}_p(\mathbb{D}_i(t)) = \bm{K}_p(\mathbb{D}_i(t)) + \bar{w}_p^2 \bm{I}_{| \mathbb{D}_i(t) |}$ with $\bm{K}_p(\mathbb{D}_i(t)) = [\kappa_p(\bm{x}_i(t_{i,k_1}), \bm{x}_i(t_{i,k_2}))]_{k_1, k_2 = \underline{k}_i(t), \cdots, \bar{k}_i(t)}$.
	The concatenated output $\bm{\zeta}_p(\mathbb{D}_i(t))$, $\forall p = 1, \cdots, n$ is defined as $\bm{\zeta}_p(\mathbb{D}_i(t)) = [y_{i,p}(t_{i, \underline{k}_i(t)}), \cdots, y_{i,p}(t_{i, \bar{k}_i(t)})]^T$.
	As a calibrated learning-based model, GP provides the posterior mean $\mu_p(\cdot, \mathbb{D}_i(t))$ as the estimation of $d_p(\cdot)$ on agent $i \in \mathcal{V}$ at time $t \in \mathbb{R}_{0,+}$, and the posterior variance $\sigma_p^2(\cdot, \mathbb{D}_i(t))$ for prediction performance quantification as follows.
	
	\begin{lemma} \label{lemma_GP_error_bound}
		Apply GP regression on agent $i \in \mathcal{V}$ to estimate$\bm{d}(\cdot)$ satisfying \cref{assumption_GP,assumption_data_set} with $| \mathbb{D}_i(t) | \le \bar{N}$.
		Then, the prediction error of concatenated $\bm{\mu}(\bm{x}, \mathbb{D}) = [\mu_1(\bm{x}, \mathbb{D}), \cdots, \mu_n(\bm{x}, \mathbb{D})]^T$ is bounded by
		\begin{align} \label{eqn_prediction_error_bound}
			\| \bm{d}(\bm{x}_i(t)) - \bm{\mu}(\bm{x}_i(t), \mathbb{D}_i(t)) \| \le \beta \sigma(\bm{x}_i(t), \mathbb{D}_i(t))
		\end{align}
		with $\bm{x}_i(t) \in \mathbb{X}$ for any $t \in \mathbb{R}_{0,+}$, where $\beta = (\sum_{p = 1}^n (\bar{d}_p^2 + \bar{N}))^{1/2}$ and $\sigma(\bm{x}, \mathbb{D}) = \| [\sigma_1(\bm{x}, \mathbb{D}), \cdots, \sigma_n(\bm{x}, \mathbb{D})]^T \|$.
	\end{lemma}
	\begin{IEEEproof}
		See appendix.
	\end{IEEEproof}
	
	\begin{remark}
		The prediction error bound in \cref{lemma_GP_error_bound} is deterministic and derived from RKHS properties under \cref{assumption_data_set}, rather than a probabilistic upper confidence bound \cite{srinivas2012information}.
		Instead of a confidence level, the coefficient $\beta$ represents a worst-case upper bound depending on the maximal size $\bar{N}$ of $\mathbb{D}$.
		Although $\beta$ increases with the data set size due to the use of a distribution-independent bound, the variance $\sigma(\cdot,\cdot)$ decreases as more informative samples are collected via online learning \cite{williams2006gaussian}. 
		Consequently, the overall prediction error bound remains meaningful and captures the improvement in learning accuracy through the reduction of uncertainty.
	\end{remark}
	
	Moreover, we consider that the computational time of GP is bounded and satisfies the following assumption.
	
	\begin{assumption} \label{assumption_computation_time}
		There exist well-defined constants $\underline{\Delta}_i \in \mathbb{R}_{0,+}$ for all $i \in \mathcal{V}$, such that $\Delta_i(t_{i,k}) \le \underline{\Delta}_i$ for any $k \in \mathbb{N}$.
	\end{assumption}
	
	The bounded $\Delta_i(\cdot)$ in \cref{assumption_computation_time} is resulted from the bounded size of $\mathbb{D}_i(\cdot)$in \eqref{eqn_data_set}, i.e., $| \mathbb{D}_i(t_{i,k}) | \le \bar{N}$, $\forall i \in \mathcal{V}$ and $k \in \mathbb{N}$. 
	In particular, using incremental Cholesky rank-one updates allows a computational complexity $\Delta_i(\cdot) \sim \mathcal{O}(|\mathbb{D}_i(\cdot)|^2)$ for online GP \cite{williams2006gaussian}. 
	Consequently, the computation time remains uniformly bounded as in \cref{assumption_computation_time}.
	While online learning continuously collects new data samples, the data set construction in \eqref{eqn_data_set} incorporates data deletion to maintain the bounded $| \mathbb{D}_i(\cdot) |$. 
	Specifically, if $| \mathbb{D}_i(t_{i,k}^-) | = \bar{N}$, where $t_{i,k}^-$ denotes the time instant immediately before the model update at $t_{i,k}$, the oldest data pair is removed before adding the newly collected sample.
	In practice, computational efficiency can be improved using e.g., LoG-GP \cite{lederer2021gaussian} and Sky-GP \cite{yang2026streaming}.
	\looseness=-1 
	
	\begin{remark}
		The proposed framework explicitly models computational delays induced by online GP inference and model updates with complexity $\mathcal{O}(\bar{N}^2)$. 
		Communication delays ($\sim \!\! 1$ms, \cite{zhang2020modular}) between agents are assumed negligible relative to the computation delay ($>\!\!50$ms, \cite{yang2026streaming}). 
		Extending the framework to simultaneously account for both computational and communication delays is considered as future work.
		\looseness=-1
	\end{remark}
	
	For notational simplicity, define $\bm{\mu}_i\!(k) \!\!=\!\! \bm{\mu}(\bm{x}_i\!(t_{i, k}),\! \mathbb{D}_i\!(t_{i, k}))$, because $\bm{x}_i$ and $\mathbb{D}_i$ always have the same time stamp.
	
	%%%%%%%%%%%%%%%%%%%%%%%%%%%%%%%%%%%%%%%%%%%%%%%%%%%%%%%%%%%%%%%%%%%
	\subsection{Cooperative Learning with Delayed Predictions }
	\label{subsection_aggregation_learning}	
	
	Before presenting the proposed asynchronous cooperative learning strategy, we first analyze the inference accuracy of predictions to answer \textbf{Q1}. 
	Specifically, we examine how well the prediction $\bm{\mu}_j(k_j(t))$ of agent $j \in {i} \cup \mathcal{N}_i$ approximates the true value of the unknown function $\bm{d}(\bm{x}_i(t))$ at time $t$.
	\begin{lemma} \label{lemma_delayed_prediction}
		Assume that GP models on all agents $i \in \mathcal{V}$ satisfy \cref{assumption_GP,assumption_data_set}. 
		Then, the prediction error $\| \bm{d}(\bm{x}_i(t)) \!-\! \bm{\mu}_j(k_j(t)) \|$, $\forall t \in \mathbb{R}_{0,+}$ is bounded by
		\begin{align} \label{eqn_delayed_prediction_error_bound}
			\| \bm{d}(\bm{x}_i\!(t)) \!\!-\!\! \bm{\mu}_j(k_j\!(t)) \| \!\!\le\!\! \eta_{i,j}\!(t) \!\!=\!\! \beta \sigma_j\!(k_j(t)) \!\!+\!\! L_d \!\|\! \tilde{\bm{x}}_{i,j}\!(t) \!\|^{\!\frac{1}{2}\!},\!
		\end{align}
		where $L_d = ( 2 \sum\nolimits_{p = 1}^n \bar{d}_p^2 L_{\kappa, p} )^{\frac{1}{2}}$, $\tilde{\bm{x}}_{i,j}(t) = \bm{x}_i(t) - \bm{x}_j(t_{j, k_j(t)})$ and $\sigma_j(k_j(t)) = \sigma(\bm{x}_j(t_{j, k_j(t)}), \mathbb{D}_j(t_{j, k_j(t)}))$.
	\end{lemma}
	\begin{IEEEproof}
		See appendix.
	\end{IEEEproof}
	
	The bound $\eta_{i,j}(t)$ in \cref{lemma_delayed_prediction} extends the standard GP prediction error bound in \cref{lemma_GP_error_bound} by incorporating an additional term which accounts for misalignment between agents. 
	
	\begin{remark}
		The square-root dependence on the state mismatch $\tilde{\bm{x}}_{i,j}$ in \cref{lemma_delayed_prediction} originates from the RKHS norm bound in \cref{assumption_GP}, and differs from linear Lipschitz bounds. 
		In contrast to classical Lipschitz conditions, \cref{lemma_delayed_prediction} implies Hölder continuity of order $1/2$ for $\bm{d}(\cdot)$.
		This relaxes the requirement of bounded derivatives, which is practically difficult to obtain.
		Moreover,this square-root form leads to sublinear growth for large mismatches, since $\| \tilde{\bm{x}}{i,j} \|^{1/2} < \| \tilde{\bm{x}}{i,j} \|$ when $\| \tilde{\bm{x}}{i,j} \| > 1$.
		However, it also implies a slower decay rate near equilibrium, as $\| \tilde{\bm{x}}{i,j} \|^{1/2} > \| \tilde{\bm{x}}{i,j} \|$ when $\| \tilde{\bm{x}}{i,j} \| < 1$.
		Consequently, the convergence rate of the error bound becomes locally slower.
	\end{remark}
	
	With the saved predictions $\bm{\mu}_j(k_j(t))$ from agents $j \in  \mathcal{N}_i$, the compensation $\hat{\bm{d}}_i(t)$ in \eqref{eqn_control_law} is designed as the aggregation of all stored predictions to address \textbf{Q2}, which is written as
	\begin{align} \label{eqn_hat_d_i}
		\hat{\bm{d}}_i(t) = \sum\nolimits_{j \in \{ i, \mathcal{N}_i \}} \omega_{i,j}(t) \bm{\mu}_j(k_j(t)),
	\end{align} 
	where $\omega_{i,j}(\cdot): \mathbb{R}_{0,+} \to [0,1] \subset \mathbb{R}$ are the time-varying aggregation weights for the prediction from agent $j$ used on agent $i$.
	In this paper, the construction of $\omega_{i,j}(\cdot)$ is inspired by the generalized Product-of-Experts (gPoE, \cite{cao2014generalized}), where predictions associated with smaller GP uncertainty and smaller query-point mismatch receive larger aggregation weights. 
	Specifically, the aggregation weight $\omega_{i,j}(\cdot)$ is designed as
	\begin{align} \label{eqn_aggregation_weight}
		\omega_{i,j}(t) = \frac{\omega_{i,j}^* \rho_{i,j}(t) \eta_{i,j}^{-2}(t)}{ \sum_{s \in \{ i, \mathcal{N}_i \}} \omega_{i,s}^* \rho_{i,s}(t) \eta_{i,s}^{-2}(t) },
	\end{align}
	where $\omega_{i,j}^* \in \mathbb{R}_{0,+}$ are pre-defined constants to show the prior confidence of GP models $j$ on agent $i$.
	The term $\rho_{i,j}(t) = \gamma(\eta_{i,j}(t) - \beta \sigma_s)$, $\forall i, j \in \mathcal{V}$ with $\sigma_s = (\sum_{p = 1}^n \sigma_p^2)^{\frac{1}{2}}$ and $\bar{\eta} \in \mathbb{R}_+$ reflects the informativeness \cite{cao2014generalized}, where $\gamma(\cdot): \mathbb{R} \to \{0, 1\}$ is the Heaviside step function.
	
	\begin{remark}
		The weights $\eta_{i,j}(t)$ are directly computable online for all $t \in \mathbb{R}_{0,+}$, since it depends only on the locally available state $\bm{x}_i(t)$ and the delayed neighboring information ${\bm{x}_j(t_{j,k_j(t)}), \sigma_j(k_j(t))}$ transmitted by agent $j \in \mathcal{N}_i$.
	\end{remark}
	
	To quantify the accuracy of the aggregated prediction, we present the following theorem. 
	
	\begin{theorem} \label{theorem_delayed_agent_aggregated_prediction}
		Let all Gaussian process models in all agents satisfy \cref{assumption_data_set,assumption_GP}, and apply the aggregation strategy as in \eqref{eqn_hat_d_i} and \eqref{eqn_aggregation_weight}.
		Then, the prediction error of each agent $i \in \mathcal{V}$ is bounded for any $t \in \mathbb{R}_{0,+}$ by
		\begin{align} \label{eqn_GP_error_bound}
			\|\! \bm{d}(\bm{x}_i\!(t)\!) \!\!-\!\! \hat{\bm{d}}_i\!(t) \!\| \!\!\le\!\! \tilde{\eta}_i\!(t) \!\!=\!\! \omega_i^* \!\big(\! \sum\nolimits_{j \!\in\! \{\! i, \!\mathcal{N}_i \!\}} \!\! \omega_{i,j}^* \rho_{i,\!j}(t) \eta_{i,\!j}^{\!-\!2}\!(t) \big)^{\!\!-\!\frac{1}{2}},\!
		\end{align}
		where $\omega_i^* = ( \sum\nolimits_{j \in \{ i, \mathcal{N}_i \}} \omega_{i,j}^* )^{1/2}$.
	\end{theorem}
	\begin{IEEEproof}
		See appendix.
	\end{IEEEproof}
	
	\cref{theorem_delayed_agent_aggregated_prediction} shows a time-varying estimation error bound $\tilde{\eta}_i(\cdot)$ for the aggregated prediction, whose value is related to $\eta_{i,j}(\cdot)$ from agent $i$ and its neighbors $j \in \mathcal{N}_i$ for all $i \in \mathcal{V}$. 
	It is important to note that \cref{theorem_delayed_agent_aggregated_prediction} does not aim to establish posterior contraction or asymptotic convergence of GP models. 
	Such results typically rely on aggregation at identical query points or synchronous prediction, which are not applicable under the heterogeneous asynchronous computational delays considered in this work. 
	Instead, \cref{theorem_delayed_agent_aggregated_prediction} provides a control-oriented guarantee that directly enables the stability and performance analysis in \cref{section_controller}.

	%%%%%%%%%%%%%%%%%%%%%%%%%%%%%%%%%%%%%%%%%%%%%%%%%%%%%%%%%%%%%%%%%%%
	\section{Distributed Second-Order Consensus Control}
	\label{section_controller}
	
	This section answers the question \textbf{Q3} about the control law design.
	Specifically, the correction term $\bm{\tau}_{c,i}(\cdot)$, $\forall i \in \mathcal{V}$ in \eqref{eqn_control_law} is devised in \cref{subsection_adjoint_MAS}.
	Then, the performance of MAS is analyzed in \cref{subsection_performance} to answer \textbf{Q4}.
	
	%%%%%%%%%%%%%%%%%%%%%%%%%%%%%%%%%%%%%%%%%%%%%%%%%%%%%%%%%%%%%%%%%%%
	\subsection{Adjoint Multi-agent System and Correction Input}
	\label{subsection_adjoint_MAS}
	%%%%%%%%%%%%%%%%%%%%%%%%%%%%%%%%%%%%%%%%%%%%%%%%%%%%%%%%%%%%%%%%%%%
	In this subsection, the control law \eqref{eqn_control_law} is completed by specifying the correction term $\bm{\tau}_{c,i}(\cdot)$ for all agents $i \in \mathcal{V}$. 
	We design the control correction input term $\bm{\tau}_{c,i}(\cdot)$ by constructing an adjoint MAS composed of $N$ agents as
	\begin{align} \label{eqn_adjoint_dynamics}
		\ddot{\bar{\bm{q}}}_i = - \bar{c}_0 \bar{\bm{e}}_{q,i} - \bar{c}_1 \dot{\bar{\bm{e}}}_{q,i} - \bar{c}_3 \dot{\bar{\bm{q}}}_i
	\end{align}
	with joint agent coordinate $\bar{\bm{q}}_i \in \mathbb{X}_q \subset \mathbb{R}^n$, adjoint consensus error defined by $\bar{\bm{e}}_{q,i} = \sum_{j \in \mathcal{N}_i} (\bar{\bm{q}}_i - \bar{\bm{q}}_j)$, and the initial state values are $\bar{\bm{q}}_i(0) = \bm{q}_i(0)$, $\dot{\bar{\bm{q}}}_i(0) = \dot{\bm{q}}_i(0)$.
	The choice of gains $\bar{c}_0, \bar{c}_1, \bar{c}_3 \in \mathbb{R}$ determines the performance of the adjoint system \eqref{eqn_adjoint_dynamics} as shown below.
	
	\begin{lemma} \label{remark_condition_c_bar}
		Consider an adjoint MAS \eqref{eqn_adjoint_dynamics} with $N$ agents and any topology.
		Choose $\bar{c}_0, \bar{c}_1, \bar{c}_3 \in \mathbb{R}_+$ satisfying $\bar{c}_1 \bar{c}_3 \ge \bar{c}_0$, then the adjoint MAS \eqref{eqn_adjoint_dynamics} achieves second-order consensus defined in \cref{definition_consensus} with $\epsilon_0 = \epsilon_1 = 0$.
	\end{lemma}
	\begin{IEEEproof}
		See appendix.
	\end{IEEEproof}

	Due to the absence of $\bm{d}(\cdot)$ in \eqref{eqn_adjoint_dynamics}, the adjoint MAS achieves second-order consensus asymptotically as shown in \cref{remark_condition_c_bar}.
	Considering the introduced adjoint MAS, we design the correction term $\bm{\tau}_{c,i}(\cdot)$ in \eqref{eqn_control_law} for each agent $i$ as follows
	\begin{align} \label{eqn_control_law_correction}
		\bm{\tau}_{c,i}(t) &= \bm{M}(\bm{q}_i(t)) \big(\ddot{\bar{\bm{q}}}_i(t) + c_0 \bar{\bm{e}}_{q,i}(t) + c_1 \dot{\bar{\bm{e}}}_{q,i}(t) \\
		&\quad + c_2 \bar{\bm{q}}_i(t) + c_3 \dot{\bar{\bm{q}}}_i(t) \big) - \bm{\tau}_{\varepsilon,i}(t) \nonumber
	\end{align}
	with positive constant $c_2 \in \mathbb{R}_+$ and
	\begin{align} \label{eqn_tau_epsilon_i}
		\bm{\tau}_{\varepsilon,i}\!(t) =& L_d \omega_i^* \mathrm{sign}(\bm{M}^{-\!T\!}\!(\bm{q}_i) \dot{\bm{\varepsilon}}_{q,i}\!) \\
		&\times \sum\nolimits_{j \in \{ i, \mathcal{N}_i \}} (\| \bm{x}_i(t) - \bm{x}_j(t) \| / \omega_{i,j}^*)^{\frac{1}{2}},
	\end{align}
	where $\bm{\varepsilon}_{q,i} = \bm{q}_i - \bar{\bm{q}}_i$ and $\bar{\bm{x}}_i(t) = [\bar{\bm{q}}_i^T(t), \dot{\bar{\bm{q}}}_i^T(t)]^T$.

	%%%%%%%%%%%%%%%%%%%%%%%%%%%%%%%%%%%%%%%%%%%%%%%%%%%%%%%%%%%%%%%%%%%
	\subsection{Analysis of Learning-based Control Performance}
	\label{subsection_performance}
	%%%%%%%%%%%%%%%%%%%%%%%%%%%%%%%%%%%%%%%%%%%%%%%%%%%%%%%%%%%%%%%%%%%
	
	The control performance is investigated for the proposed controller \eqref{eqn_control_law} to answer \textbf{Q4} in \cref{section_introduction}.
	The introduced adjoint MAS serves as an ideal reference dynamics, assuming the unknown $\bm{d}(\cdot)$ is exactly canceled. 
	The proposed learning-based controller \eqref{eqn_control_law} then seeks to drive the actual system toward this ideal reference behavior despite delayed GP predictions. 
	This construction is conceptually related to certainty-equivalence control \cite{karafyllis2018adaptive} and backstepping \cite{farrell2009command}, where stability analysis is performed through the tracking error between the actual and adjoint systems.
	Specifically, considering the MAS \eqref{eqn_agent_EL_dynamics} with the proposed control law \eqref{eqn_control_law}, \eqref{eqn_control_law_consensus}, \eqref{eqn_control_law_model} and \eqref{eqn_control_law_correction}, the controlled agent dynamics of the agent $i \in \mathcal{V}$ is written as $\ddot{\bm{q}}_i = \ddot{\bar{\bm{q}}}_i - c_0 \tilde{\bm{e}}_{q,i} - c_2 \bm{\varepsilon}_{q,i} - c_1 \dot{\tilde{\bm{e}}}_{q,i} - c_3 \dot{\bm{\varepsilon}}_{q,i} + \bm{M}^{-1}(\bm{q}_i) (\bm{d}(\bm{x}_i) - \hat{\bm{d}}_i(t) - \bm{\tau}_{\varepsilon,i}(t))$,	for all $i \in \mathcal{V}$, where $\bm{\varepsilon}_{q,i} = \bm{q}_i - \bar{\bm{q}}_i$ and $\tilde{\bm{e}}_{q,i} = \bm{e}_{q,i} - \bar{\bm{e}}_{q,i} = \sum\nolimits_{j \in \mathcal{N}_i} (\bm{\varepsilon}_{q,i} - \bm{\varepsilon}_{q,j})$.
	Moreover, denote the concatenation of $\bm{\varepsilon}_{q,i}$, $\forall i \!\in\! \mathcal{V}$ by $\bm{\varepsilon}_q \!\!=\!\! [\bm{\varepsilon}_{q,1}^T, \!\cdots\!, \bm{\varepsilon}_{q,N}^T]^T \!\!\in\!\! \mathbb{R}^{n N}$, whose dynamics is $\ddot{\bm{\varepsilon}}_q = - ((c_0 \bm{L} + c_2 \bm{I}_N) \otimes \bm{I}_n) \bm{\varepsilon}_q - ((c_1 \bm{L} + c_3 \bm{I}_N) \otimes \bm{I}_n) \dot{\bm{\varepsilon}}_q + \bm{M}_{\mathcal{G}}^{-1}(\bm{q})(\bm{d}_{\mathcal{G}}(\bm{x}) - \hat{\bm{d}}_{\mathcal{G}}(t) - \bm{\tau}_{\varepsilon,\mathcal{G}}(t))$
	with $\bm{M}_{\mathcal{G}}(\bm{q}) = \mathrm{blkdiag}(\bm{M}(\bm{q}_1), \cdots, \bm{M}(\bm{q}_N)) \in \mathbb{R}^{nN \times nN}$, $\bm{d}_{\mathcal{G}}(\bm{x}) \!\!=\!\! [\bm{d}^T(\bm{x}_1), \!\cdots\!, \bm{d}^T(\bm{x}_N)]^T \!\in\! \mathbb{R}^{nN}$, $\hat{\bm{d}}_{\mathcal{G}} \!=\! [\hat{\bm{d}}_1^T, \!\cdots\!, \hat{\bm{d}}_N^T]^T \!\in\! \mathbb{R}^{nN}$ and $\bm{\tau}_{\varepsilon,\mathcal{G}} \!=\! [\bm{\tau}_{\varepsilon,1}^T, \!\cdots\!, \bm{\tau}_{\varepsilon,N}^T]^T \!\in\! \mathbb{R}^{nN}$.
	Considering that $\bm{\varepsilon}_i$ in \eqref{eqn_tau_epsilon_i} is also written as $\bm{\varepsilon}_i = [\bm{\varepsilon}_{q,i}^T, \dot{\bm{\varepsilon}}_{q,i}^T]^T$, the dynamics of the concatenation for $\bm{\varepsilon}_i$, i.e., $\bm{\varepsilon} = [\bm{\varepsilon}_1^T, \cdots, \bm{\varepsilon}_N^T]^T \in \mathbb{R}^{2 n N}$, is
	\begin{align} \label{eqn_epsilon_dynamics}
		\dot{\bm{\varepsilon}}(t) \!=\! \bm{\Phi} \bm{\varepsilon}(t) \!+\! \bm{B}(\bm{q}(t)) ( \bm{d}_{\mathcal{G}}(\bm{x}(t)) \!-\! \hat{\bm{d}}_{\mathcal{G}}(t) \!-\! \bm{\tau}_{\varepsilon,\mathcal{G}}(t) ),
	\end{align}
	where $\bm{\Phi} \!\!=\!\! \bm{\Phi}_s \!\otimes\!\! \bm{I}_n \!\!\in\!\! \mathbb{R}^{2 n N \times 2 n N}$, $\bm{B}(\bm{q}) \!\!=\!\! \big[ \bm{0}_{n N \!\times\! n N},\! \bm{M}_{\mathcal{G}}^{-T}(\bm{q}) \big]^T \!\in\! \mathbb{R}^{2 n N \times n N}$ 
	and $\bm{\Phi}_s \!\!=\!\! \begin{bmatrix}
		\bm{0}_{N \times N} & \bm{I}_N \\
		\!-\! c_0 \bm{L} \!\!-\!\! c_2 \bm{I}_N & - c_1 \bm{L} \!\!-\!\! c_3 \bm{I}_N
	\end{bmatrix}$.
	The gains $c_i$ with $i = 0, \cdots, 3$ are chosen as follows.
	
	\begin{lemma} \label{remark_condition_c}
		Choose $c_0,\! c_1,\! c_2,\! c_3 \!\in\! \mathbb{R}_+$, such that $c_3^2 \!\ge\! 4c_2$ and $2 c_0 \!\le\! c_1 \!(\! c_3 \!+\! (c_3^2 \!-\! 4 c_2)^{\frac{1}{2}} )$, then $\bm{\Phi}$ in \eqref{eqn_epsilon_dynamics} is negative definite.
	\end{lemma}
	\begin{IEEEproof}
		See appendix.
	\end{IEEEproof}
	
	With the property of $\bm{\Phi}$ in \cref{remark_condition_c}, the convergence of the dynamics \eqref{eqn_epsilon_dynamics} is shown in the following lemma.
	
	\begin{lemma} \label{lemma_tracking_error_bound}
		Consider a controlled MAS \eqref{eqn_agent_EL_dynamics} satisfying \cref{assumption_dynamics,assumption_graph,assumption_data_set,assumption_GP,assumption_computation_time}, where the coefficients $c_i$ with $i = 0, \cdots, 3$ and $\bar{c}_i$ with $i = 0, 1, 3$ satisfy \cref{remark_condition_c_bar,remark_condition_c}, respectively.
		Then, the value of $\bm{\varepsilon}_i$, $\forall i \in \mathcal{V}$ is ultimately bounded by
		\begin{align}
			\lim\nolimits_{t \to \infty} \| \bm{\varepsilon}_i(t) \| \le \bar{\varepsilon} = 2 \underline{\lambda}^{-1} \underline{M} \| \bm{\theta} \|,
		\end{align}
		where $\bm{\theta} = [\theta_1, \cdots, \theta_N]^T$.
		The entries $\theta_i \in \mathbb{R}$ for $i \in \mathcal{V}$ denote
		\begin{align} \label{eqn_theta_i}
			\theta_i = \omega_i^* \big( \sum\nolimits_{j \in \{ i, \mathcal{N}_i \}} \omega_{i,j}^* (\beta \bar{\sigma}_j + L_d (2 F \underline{\Delta}_j)^{\frac{1}{2}})^{-2} \big)^{-\frac{1}{2}},
		\end{align}
		where $\underline{\lambda} \in \mathbb{R}_+$ denote the minimal eigenvalue of $- \bm{\Phi}$ and $\bar{\sigma}_j = (\sum_{p=1}^n \bar{w}_p^2)^{\frac{1}{2}}$ is the upper bound of posterior variance.
		The constant $F \in \mathbb{R}_+$ is defined as $F = (2 \varepsilon_{\max,0} + \varepsilon_{\max,1}^2) \| \bm{\Phi}_s \| + \underline{M} ( \hat{\omega} \bar{\eta} + \bar{\tau}_{\varepsilon,\mathcal{G}}) + \bar{q}_{\max}$, where $\varepsilon_{\max,0} = \upsilon_s \max\{ \hat{\omega} \underline{M} ( \bar{\eta} + 2 (n \| \bm{x}(0) \|)^{1/2}) / \underline{\lambda}, \| \bm{\varepsilon}(0) \| \}$, $\varepsilon_{\max,1} = 2 \hat{\omega} \upsilon_s \underline{M} n^{1/2} / \underline{\lambda}$, $\hat{\omega} = \sum\nolimits_{i \in \mathcal{V}} \sum\nolimits_{j \in \{ i, \mathcal{N}_i \}} \hat{\omega}_{i,j}$, $\bar{\tau}_{\varepsilon,\mathcal{G}} = 2 \hat{\omega} n^{1/2} ( (2 \varepsilon_{\max,0} + \varepsilon_{\max,1}^2)^{1/2} + \| \bm{x}(0) \|^{1/2} )$ and $\upsilon_s \!= \| \bm{Q}_s \| \| \bm{Q}_s^{-1} \|$.
		The matrix $\bm{Q}_s \in \mathbb{R}^{2 N \times 2 N}$ is the orthogonal eigenvector matrix of $\bm{\Phi}_s$.
	\end{lemma}
	\begin{IEEEproof}
		See appendix.
	\end{IEEEproof}
	
	\cref{lemma_tracking_error_bound} shows the ultimate bound of $\| \bm{\varepsilon}_i(\cdot) \|$, $\forall i \in \mathcal{V}$.
	Considering $\bar{\bm{x}}_i(0) = \bm{x}_i(0)$, it holds $\| \bm{\varepsilon}_i(t) \| \le \bar{\varepsilon}$, $\forall t \in \mathbb{R}_{0,+}$.
	
	\begin{remark}
		The Lyapunov function adopted in \cref{lemma_tracking_error_bound} is chosen for analytical tractability and explicit characterization of the effects of GP uncertainty, asynchronous delays and cooperative learning. 
		Less conservative estimates may potentially be obtained using matrix-weighted quadratic Lyapunov functions associated with the system dynamics \eqref{eqn_agent_EL_dynamics}. 
		Investigating such tighter bounds is an interesting direction for future work.
	\end{remark}
	
	\begin{remark}
		The discontinuous sign function in $\tau_{\varepsilon,i}$ introduces chattering effects.
		The amplitude of this term is related to the state mismatch between agents, which satisfies
		\begin{align}
			\| \bm{x}_i(t) - \bm{x}_j(t) \| \le& \| \bm{\varepsilon}_i(t) \| + \| \bm{\varepsilon}_j(t) \| + \| \bar{\bm{x}}_i(t) - \bar{\bm{x}}_j(t) \| \nonumber \\
			\le& 2 \bar{\varepsilon} + \| \bar{\bm{x}}_i(t) - \bar{\bm{x}}_j(t) \|,
		\end{align}
		considering \cref{lemma_tracking_error_bound}.
		This also leads to $\lim_{t \to \infty} \| \bm{x}_i(t) - \bm{x}_j(t) \| \le 2 \bar{\varepsilon}$ due to the asymptotic consensus in \cref{remark_condition_c_bar}.
		Moreover, the value of $\bar{\varepsilon}$ depends on the prediction error bound $\beta \bar{\sigma}_i$ and the computational delay $\underline{\Delta}_i$, indicating that improved learning accuracy and reduced delay lead to smoother control inputs.
		To reduce chattering, $\mathrm{sign}(\cdot)$ can be replaced by smooth approximations, such as saturation \cite{utkin2013sliding} or hyperbolic tangent functions \cite{wang2024non}.
		These methods introduce a small residual approximation error near zero, which slightly relax the robustness while improving control smoothness.
	\end{remark}
	
	With \cref{remark_condition_c_bar,lemma_tracking_error_bound}, we are ready to present the performance analysis of the proposed distributed control law with asynchronous cooperative online GP regression as follows.
	
	\begin{theorem} \label{theorem_performance}
		Consider a MAS with unknown dynamics \eqref{eqn_agent_EL_dynamics} satisfying \cref{assumption_dynamics,assumption_graph}, where the unknown nonlinearity is estimated using the proposed asynchronous cooperative learning approach \eqref{eqn_hat_d_i} with GP models satisfying \cref{assumption_data_set,assumption_GP,assumption_computation_time}.
		If the MAS is controlled by the proposed control law \eqref{eqn_control_law} with \eqref{eqn_control_law_consensus}, \eqref{eqn_control_law_model} and \eqref{eqn_control_law_correction}, which is calculated by introducing a adjoint MAS \eqref{eqn_adjoint_dynamics} with coefficients satisfying \cref{remark_condition_c_bar,remark_condition_c}, the second-order approximate consensus defined in \cref{definition_consensus} is achieved with $\epsilon_0 = 2 \bar{\varepsilon}$ and $\epsilon_1 = \bar{\varepsilon}$.
	\end{theorem}
	\begin{IEEEproof}
		See appendix.
	\end{IEEEproof}
	
	\cref{theorem_performance} shows the approximate consensus in \cref{definition_consensus} by using the proposed online GP-based controller.
	
	\begin{remark}
		The derived ultimate error bound $\varepsilon$ is established using a uniform upper bound of the online GP posterior variance $\bar{\sigma}_i$ to obtain explicit worst-case robustness guarantees under asynchronous delayed learning.
		In practice, as additional data are collected and the sampling density increases, the GP posterior variance $\sigma_i(\cdot)$ typically decreases, leading to improved prediction accuracy and smaller consensus bound $\varepsilon$. 
		Moreover, smaller computational delays $\underline{\Delta}_i$ increase the effective learning update frequency and further improve practical performance. 
		Explicitly characterizing the coupled relationship among posterior variance $\bar{\sigma}_i$ and computation delay remains an interesting direction for future investigation.
	\end{remark}
	
	\begin{remark}
		Note that the ultimate bound $\bar{\epsilon}$ in \cref{lemma_tracking_error_bound} and \cref{theorem_performance} depends on the number of agents $N$, which originates from the aggregation of individual prediction error $\bar{\sigma}_i$ bounds and computational delays $\underline{\Delta}_i$.
		For large-scale MAS, the ultimate bound can be reduced by improving GP accuracy with smaller $\bar{\sigma}_i$ and enhancing computational capabilities for smaller $\underline{\Delta}_i$.
		The dependence on $N$ can potentially be alleviated by using local input-to-state stability \cite{zeng2016convergence} or normalized consensus gains \cite{atay2010consensus}, whose analysis is beyond the scope of this paper.
		The primary focus is to address asynchronous cooperative learning with heterogeneous delays and query point mismatch.
	\end{remark}

	\subsection{Comparison to Local Learning-based Control}
	
	\subsubsection{Learning Performance}
	To compare with the local prediction with $\hat{\bm{d}}_{\text{loc},i}(t) \!\!=\!\! \sum\nolimits_{j \in \{ i, \mathcal{N}_i \}} \omega_{i,j}(t) \bm{\mu}_{\text{loc},j}(k_j(t))$, the aggregated prediction is considered as $\bm{\mu}_{\text{loc},i}(k_i(t)) = \bm{\mu}_i(k_i(t))$ and zero neighboring predictions $\bm{\mu}_{\text{loc},j}(k_j(t)) \!\!=\!\! \bm{0}_{n \times 1}$.
	Consider the prior mean as the posterior mean with empty data sets, it holds $\bm{\mu}_{\text{loc},j}(k_j(t)) \!\!=\!\! \bm{\mu}(\bm{x}_j(t_{j,k_j(t)}), \emptyset)$ with $\sigma(\bm{x}_j(t_{j,k_j(t)}), \emptyset) \!\!=\!\! \sigma_s$, which results in $\omega_{i,j}(t) \!\!=\!\! 0$. 
	The resulting prediction error bounds are denoted by $\| \bm{d}(\bm{x}_i(t)) \!-\! \hat{\bm{d}}_{\text{loc},i}(t) \| \!\le\! \tilde{\eta}_{\text{loc},i,j}(t) \!\!=\!\! \beta \sigma_s \!\!+\!\! L_d \| \bm{x}_i(t) \!-\! \bm{x}_j(t_{j,k_j(t)}) \|^{\frac{1}{2}}$, which satisfy $\eta_{i,j}(t) < \eta_{\text{loc},i,j}(t)$ due to the strict reduction of posterior variance $\sigma(\cdot, \cdot)$.
	Then, by monotonicity, it follows that $\tilde{\eta}_i(t) \!\!<\!\! \tilde{\eta}_{\text{loc},i}(t)$, $\forall t \in \mathbb{R}_{0,+}$, which shows that cooperative learning with $\hat{\bm{d}}_i(t)$ does not deteriorate the theoretical error bound compared to the local case with $\hat{\bm{d}}_{\text{loc},i}(t)$.
	Moreover, the analysis indicates that cooperative learning improves performance when neighboring predictions have smaller uncertainty or query point mismatch.

	\subsubsection{Control Performance}
	
	According to the expression of $\theta_i$ in \eqref{eqn_theta_i}, incorporating predictions from neighboring agents can reduce the tracking error bound $\bar{\varepsilon}$, especially when those predictions have lower uncertainty (smaller $\bar{\sigma}_i$) or smaller delay-induced mismatch (smaller $\underline{\Delta}_i$).
	Specifically, following the same procedure in \cref{lemma_tracking_error_bound} and \cite{dai2024decentralized}, the ultimate tracking error bound using local learning is written as
	\begin{align}
		\bar{\varepsilon}_{\text{loc}} = 2 \underline{\lambda}^{-1} \underline{M} \| \bm{\theta}_{\text{loc}} \|, && \theta_{\text{loc},i} = \beta \bar{\sigma}_i + L_d (2 F \underline{\Delta}_i)^{\frac{1}{2}}
	\end{align}
	for all $i = 1, \cdots, N$ with $\bm{\theta}_{\text{loc}} = [\theta_{\text{loc},1}, \cdots, \theta_{\text{loc},N}]^T$.
	Similarly as in \cref{theorem_performance}, the ultimate consensus error and velocity error are expressed as $\epsilon_{\text{loc},0} = 2 \bar{\varepsilon}_{\text{loc}}$ and $\epsilon_{\text{loc},1} = \bar{\varepsilon}_{\text{loc}}$.
	The comparison between $\bar{\varepsilon}$ and $\bar{\varepsilon}_{\text{loc}}$ is shown below.
	
	\begin{proposition} \label{proposition_control}
		Consider a MAS \eqref{eqn_agent_EL_dynamics} satisfying \cref{assumption_dynamics,assumption_graph}, whose dynamics is learned via \eqref{eqn_hat_d_i} satisfying \cref{assumption_data_set,assumption_GP,assumption_computation_time} and $\omega_{i,j}^* \!=\! \omega_{j,i}^*$, $\forall (i,j) \!\in\! \mathcal{E}$.
		The system is controlled by the proposed control law \eqref{eqn_control_law} with \eqref{eqn_control_law_consensus}, \eqref{eqn_control_law_model}, \eqref{eqn_hat_d_i} and \eqref{eqn_control_law_correction} satisfying \cref{remark_condition_c,remark_condition_c_bar}.
		Then, the cooperative learning achieves improved control performance compared to local learning, in the sense that $\epsilon_0 \le \epsilon_{\text{loc},0}$ and $\epsilon_1 \le \epsilon_{\text{loc},1}$.
	\end{proposition}
	\begin{IEEEproof}
		See appendix.
	\end{IEEEproof}
	
	\cref{proposition_control} shows that cooperative learning leads to smaller error bounds than local learning. 
	In summary, agents improve their prediction accuracy and therefore enhance control performance by sharing learning information. 
	
	\section{Simulation}\label{section_simulation}
	
	To demonstrate the effectiveness of the proposed approach, simulations are conducted with more details in appendix.
	
	\subsection{Simulation Setting}
	
	\begin{figure}[t]
		\centering
		\begin{subfigure}{0.24\textwidth}
			\centering
			\begin{tikzpicture}
				\draw[->, thick,>=stealth] (0,0) -- (3.7,0) node[below] {\small $x$};
				\draw[->, thick,>=stealth] (0,0) -- (0,1.8) node[left]  {\small $y$};
				\node at (-0.2,-0.2) {\small $0$};
				
				\node at (2,1.2) {\includegraphics[width=1.2cm]{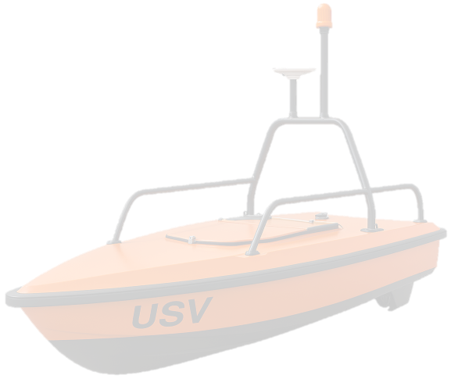}};
				\fill[red] (2,1) circle (0.7mm);
%				\node at (2,2.5) {\small Agent/USV $i$};
				
				\draw[->, red, thick,>=stealth] (0,0) -- (2,1);
				\draw[-, thick, dashed] (2,0) -- (2,1);
				\draw[-, thick, dashed] (0,1) -- (2,1);
				\node at (2,-0.2) {\small $q_{i,1}$};
				\node at (-0.3,1.5) {\small $q_{i,2}$};
				
				\draw[-, dashed, thick] (2,1) -- (3,1);
				\draw[-, dashed, thick] (2,1) -- (3,1.5);
				\draw[thick,->,>=stealth] (2.8,1) arc[start angle=0,end angle=28,radius=0.8];
				\node at (3.2,1.2) {\small $q_{i,3}$};
			\end{tikzpicture}
			\caption{}
		\end{subfigure}
		\hfill
		\begin{subfigure}{0.24\textwidth}
			\centering
			\begin{tikzpicture}
				\foreach \i in {1,2,3,4,5,6} {
					\node[circle,draw,minimum size=0.3cm, thick] (N\i) at (60*\i+0:1cm) {\scriptsize \i};
				}
				
				\foreach \i [evaluate={\j=int(mod(\i,6)+1);}] in {1,2,3,4,5,6} {
					\draw[<->,thick,>=stealth] (N\i) -- (N\j);
				}
				
			\end{tikzpicture}
			\caption{}
		\end{subfigure}
		\caption{
			Setting of the multi-USV system.
			(a) Generalized coordinate.
			(b) Communication topology $\mathcal{G}$.
		}
		\label{figure_setting}
	\end{figure}
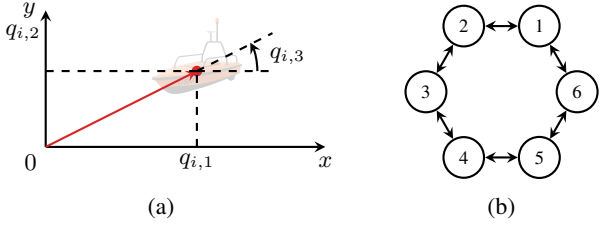
	
	In this simulation,  we investigate a second-order consensus problem involving $N = 6$ unmanned surface vehicles (USVs), as illustrated in \cref{figure_setting}. 
	Each USV is governed by the Euler-Lagrange dynamics specified in \eqref{eqn_agent_EL_dynamics}, characterized by $n = 3$ degrees of freedom (DoFs). 
	The physical parameters of all USVs are adopted from the same configuration in \cite{yang2025safe}, and each USV is subject to unknown state dependent wind effect and current influence. 
	The connection of the USVs are formulated as a graph $\mathcal{G}$ shown in \cref{figure_setting}.
	The gains are set as $\bar{c}_1 = \bar{c}_3 = c_2 = 2$, $c_1 = 1$ and $c_3 = 2 \sqrt{c_2}$, $c_0 = 0.5 c_1 c_3$, $\bar{c}_0 = \bar{c}_1 \bar{c}_3$, such that \cref{remark_condition_c_bar,remark_condition_c} hold.
	To estimate the unknown disturbance $\bm{d}(\cdot)$, GP model is deployed on each agent $i \in \mathcal{V}$ with empty initial data set $\mathbb{D}_i(0) = \emptyset$ and squared exponential kernel $\kappa_p(\bm{x}, \bm{x}') = \sigma_{f,p}^2 \exp( -\sigma_l^{-2} (\bm{x} - \bm{x}')^T (\bm{x} - \bm{x}') / 2 )$ for any $\bm{x}, \bm{x}' \in \mathbb{X}$, where $\sigma_{f,p} = 1$ for each dimension $p = 1, \cdots, n$ and $\sigma_l = 0.1$.
	Each agent $i$ collects the received data pair at $t_{i,k}$ and deletes the oldest data pair when $| \mathbb{D}_i(t) | \ge \bar{N} = 100$.
	The upper bound of measurement noise for the data pair is considered as $\bar{w}_p = 0.01$ for $p = 1, \cdots, n$.
	The compensation $\hat{\bm{d}}_i(\cdot)$ on each agent $i \in \mathcal{V}$ is calculated by the proposed distributed GPs with \eqref{eqn_hat_d_i}, where the constant coefficients in \eqref{eqn_aggregation_weight} are set as $\omega_{i,j}^* = a_{i,j} / (1 + |\mathcal{N}_i|)$ for any $i,j \in \mathcal{V}$.
	The computation time $\Delta_i(t_{i,k})$ is randomly uniformly distributed in $[0.01, 0.1]$, $\forall k \in \mathbb{N}$ and $i \in \mathcal{V}$.
	The total simulation time is set as $30$.
	To demonstrate the effectiveness of the proposed approach, some existing learning methods are compared, namely MoE, PoE, gPoE, BCM, rBCM and Local.

	\subsection{Control and Prediction Performances}
	
	\begin{figure}[t]
		\centering
		\includegraphics[width=0.48\textwidth]{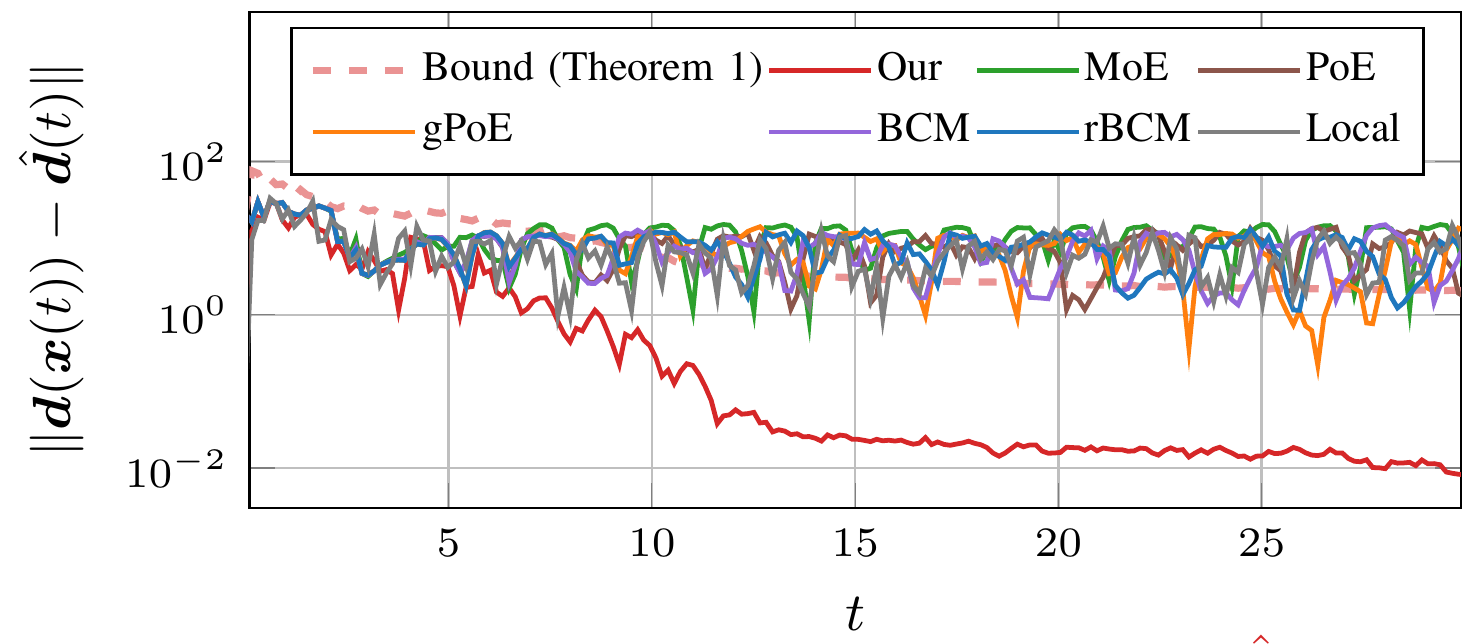}
		\caption{
			Aggregated prediction error $\| \bm{d}(\bm{x}(t)) - \hat{\bm{d}}(t) \|$ and the prediction error bound $\| \tilde{\bm{\eta}}(t) \|$ in \cref{theorem_delayed_agent_aggregated_prediction}.
		}
		\label{figure_prediction}
	\end{figure}
	
	To compare the control and prediction performances, the same random initial values $\bm{q}(0) \!=\! [\bm{q}_1^T(0), \!\cdots\!, \bm{q}_6^T(0)]^T$ are set for all methods.
	The comparative analysis of prediction accuracy across different strategies is presented in \cref{figure_prediction}, where the proposed aggregation method demonstrates superior performance by achieving significantly reduced aggregated prediction errors $\| \bm{d}(\bm{x}(t)) \!-\! \hat{\bm{d}}(t) \|$, with $\bm{d}(\bm{x}(t)) \!=\! [\bm{d}^T(\bm{x}_1(t)), \!\cdots\!, \bm{d}^T(\bm{x}_N(t))]^T$ and $\hat{\bm{d}}(t) \!=\! [\hat{\bm{d}}_1^T(t), \!\cdots\!, \hat{\bm{d}}_N^T(t)]^T$ representing the true and estimated disturbances, respectively. 
	This performance advantage can be attributed to the fundamental limitation of alternative PoE and BCM-based approaches, which neglect query point bias.
	Moreover, it is noteworthy that the cooperative learning performs better than local approach, only when the delay is explicitly considered in the aggregation as our method.
	In addition to the actual aggregated prediction error, the theoretical error bound $\tilde{\bm{\eta}}(t) = [\tilde{\eta}_1(t), \cdots, \tilde{\eta}_N(t)]^T$ derived in \cref{theorem_delayed_agent_aggregated_prediction} is also evaluated in simulation. 
	As shown in \cref{figure_prediction}, the actual prediction error $\| \bm{d}(\bm{x}(t)) - \hat{\bm{d}}(t) \|$ remains consistently below the derived bound $\| \hat{\bm{\eta}}(t) \|$ for all time, demonstrating the correctness of the theoretical analysis.
	
	\begin{figure}[t]
		\centering
		\includegraphics[width=0.48\textwidth]{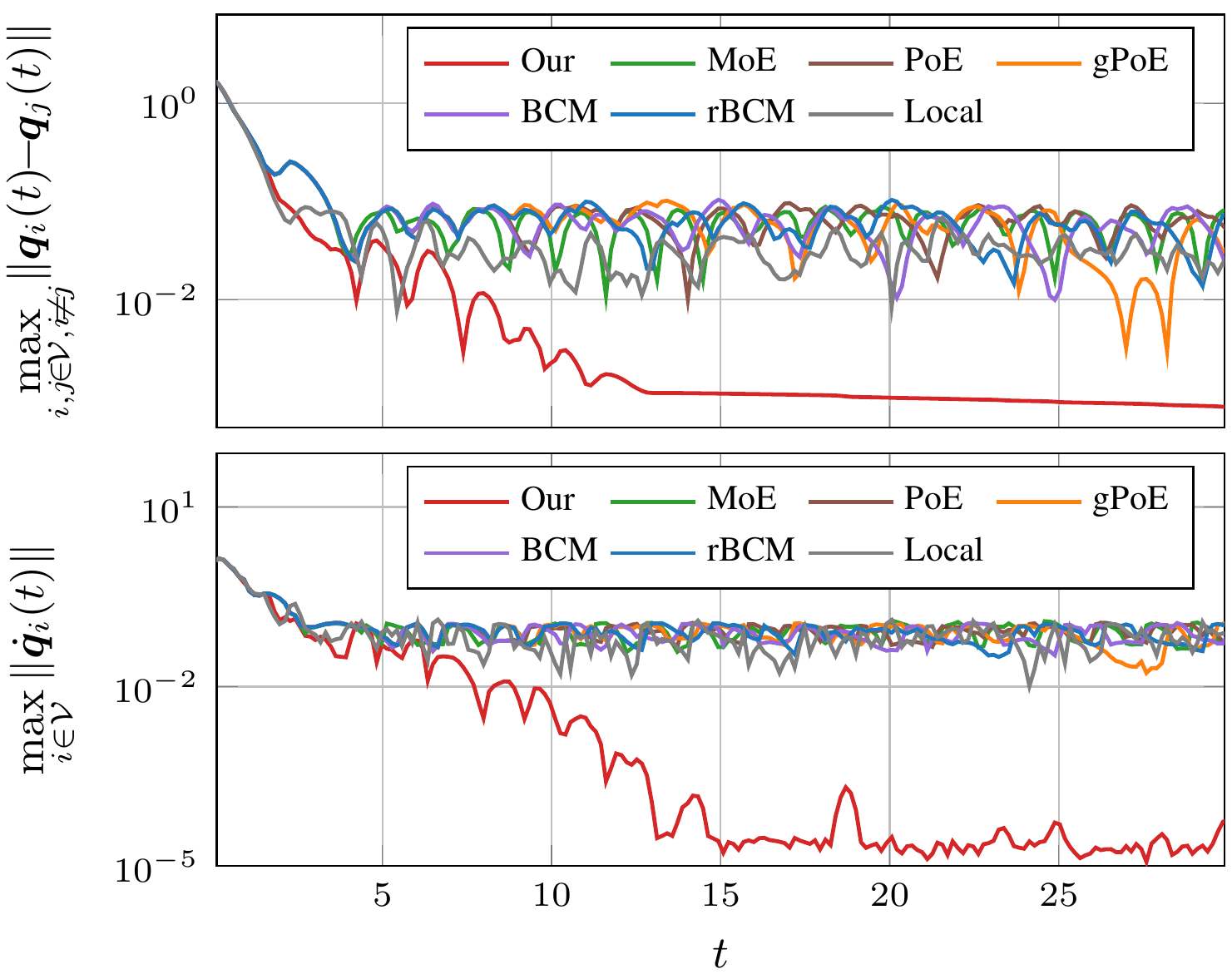}
		\caption{
			Maximal consensus error $\| \bm{q}_i(t) - \bm{q}_j(t) \|$ and velocity error $\| \dot{\bm{q}}_i(t) \|$ for all $i, j \in \mathcal{V}$ with $i \ne j$.
		}
		\label{figure_error}
	\end{figure}
	
	Given the strong interdependence between prediction accuracy and consensus performance, a comparative evaluation of control effectiveness across different cooperative learning-based strategies is conducted, with results presented in \cref{figure_error}. 
	The experimental results clearly demonstrate that the proposed methodology achieves superior performance in maximal consensus error and the maximal velocity compared to existing methods, which demonstrates the effectiveness of the proposed approach for addressing the second-order consensus problem.   

	\subsection{Monte Carlo Test}	
	
	To eliminate potential biases arising from specific initial conditions, a Monte Carlo test is conducted with 100 independent trials for each comparative method. 
	In each trial, the initial position coordinates $\bm{q}_i(0)$ and velocities $\dot{\bm{q}}_i(0)$ for each agent $i \in \mathcal{V}$ are randomly sampled from a uniform distribution over the domain $[-1,1]^3$, thereby providing a statistically representative assessment of algorithm performance across diverse initialization scenarios.	
	
	The aggregated prediction performance across all six cooperative learning-based control strategies is presented in \cref{figure_prediction_MonteCarlo}, where the proposed methodology consistently demonstrates superior prediction accuracy as reflected by smaller prediction error.
	Consequently, the enhanced disturbance estimation capability translates directly into improved second-order consensus performance as shown in \cref{figure_error_MonteCarlo}, where the maximal difference between $\bm{q}_i$ and $\bm{q}_j$ and the maximal velocity $\dot{\bm{q}}_i$ are smaller compared to existing cooperative learning frameworks, thereby demonstrating the statistical significance and robustness of the proposed approach.
	
	\begin{figure}[t] 
		\centering
		\includegraphics[width=0.48\textwidth]{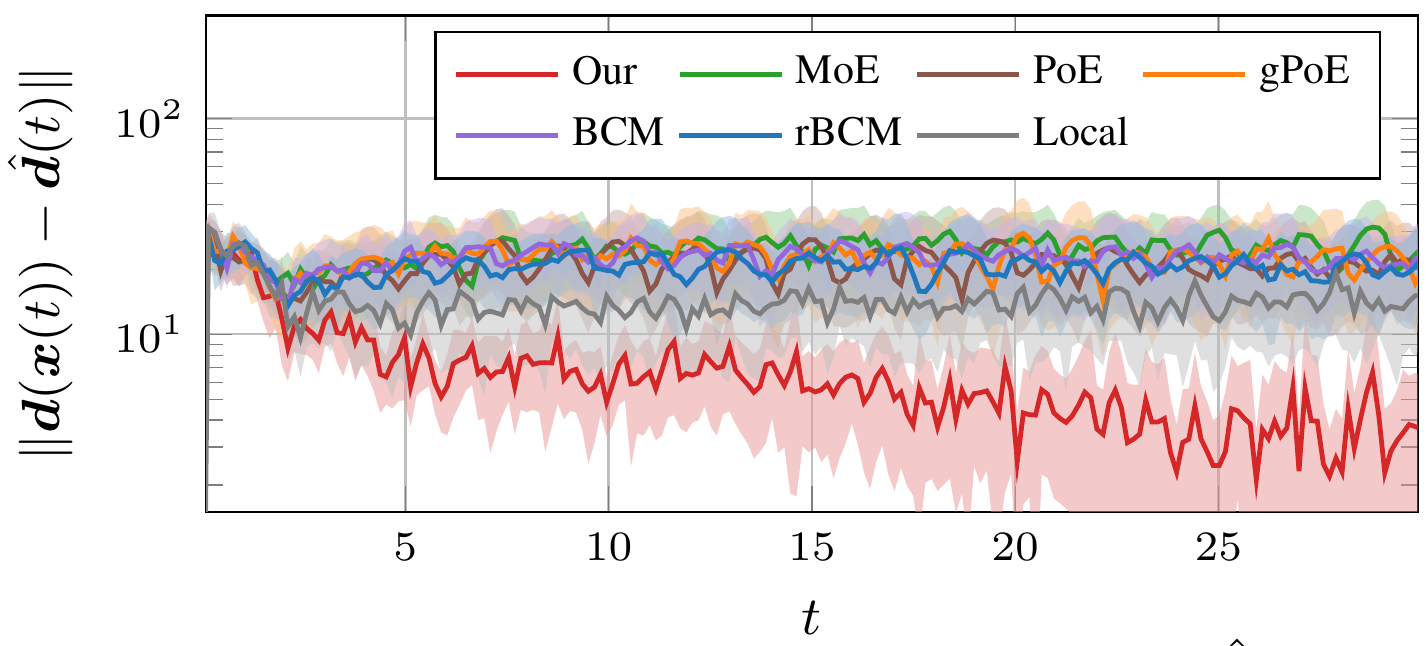}
		\caption{
			Aggregated prediction error $\| \bm{d}(\bm{x}(t)) - \hat{\bm{d}}(t) \|$ w.r.t time $t$ from $7$ cooperative learning-based control methods.
			The solid line represents the mean value among $100$ Monte Carlo tests, and the shadowed area denotes the variances.
		}
		\label{figure_prediction_MonteCarlo}
	\end{figure}
	
	\begin{figure}[t] 
		\centering
		\includegraphics[width=0.48\textwidth]{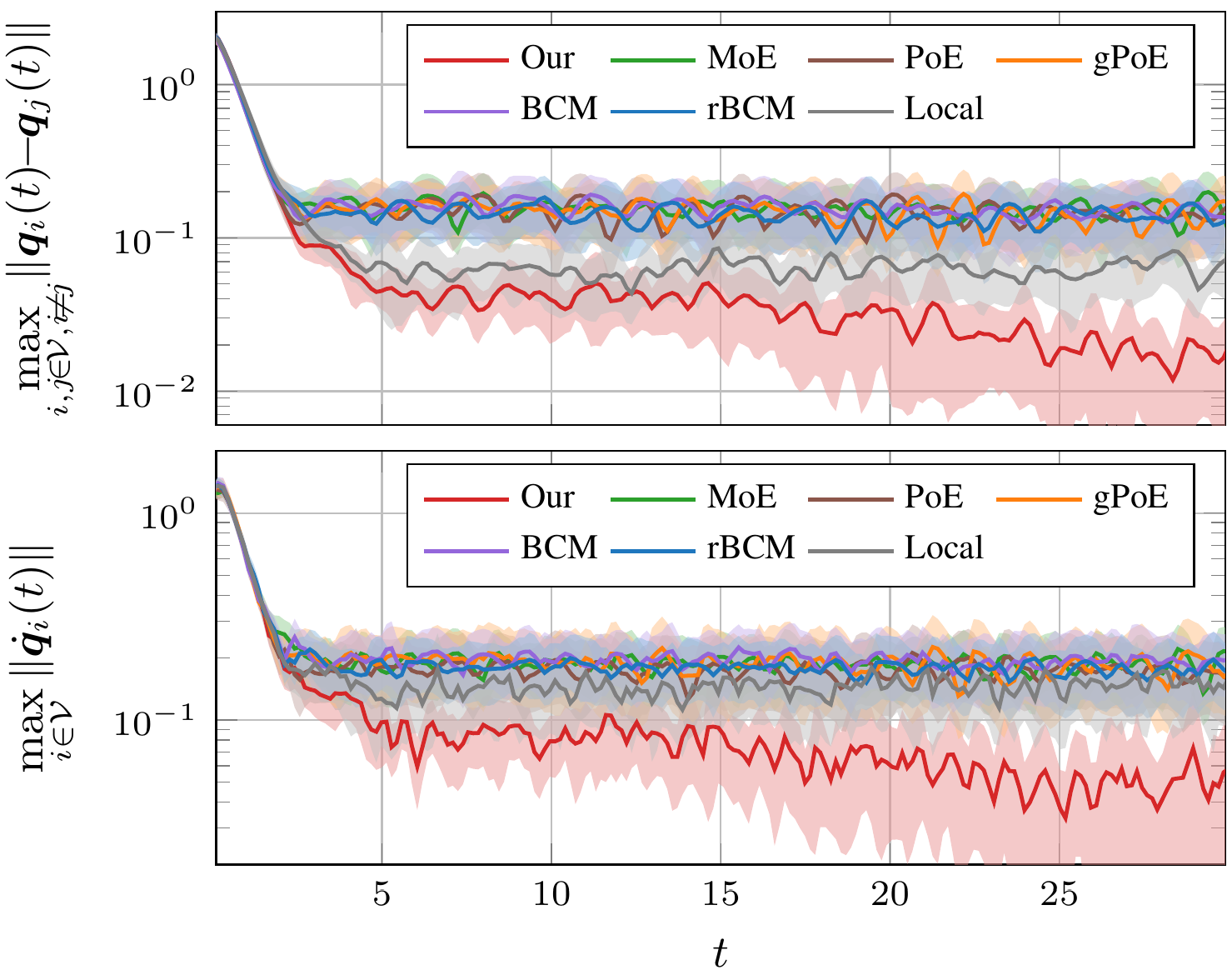}
		\caption{
			Maximal consensus error $\| \bm{q}_i(t) - \bm{q}_j(t) \|$ between any two agents $i, j \in \mathcal{V}$ with $i \ne j$ and maximal velocity error $\| \dot{\bm{q}}_i(t) \|$ for any agent $i \in \mathcal{V}$ over time $t$ using different cooperative learning-based control strategies.
			The solid line represents the mean value among $100$ Monte Carlo tests, and the shadowed area denotes the variances.
		}
		\label{figure_error_MonteCarlo}
	\end{figure}
	
	Furthermore, the evaluation of steady-state prediction and control performance is essential for comprehensively demonstrating the efficacy of the proposed methodology. To quantify steady-state behavior, three key performance indicators are employed: the maximal prediction error $\max_{t \in \mathbb{T}} \| \bm{d}(\bm{x}(t)) - \hat{\bm{d}}(t) \|$, consensus error $\max_{t \in \mathbb{T}} \max_{i, j \in \mathcal{V}, i \ne j} \| \bm{q}_i(t) - \bm{q}_j(t) \|$ and velocity error $\max_{t \in \mathbb{T}} \max_{i \in \mathcal{V}} \| \dot{\bm{q}}_i(t) \|$ in the time domain $\mathbb{T} = [10, 30]$.
	The results shown in \cref{figure_prediction_SteadyState,figure_error_SteadyState} demonstrate that the proposed approach achieves substantially reduced errors in both disturbance prediction and consensus control compared to other strategies. 
	These findings provide empirical evidence supporting the superior performance and practical effectiveness of the proposed cooperative learning-based control framework for multi-agent consensus applications.
	
	\begin{figure}[t] 
		\centering
		\includegraphics[width=0.48\textwidth]{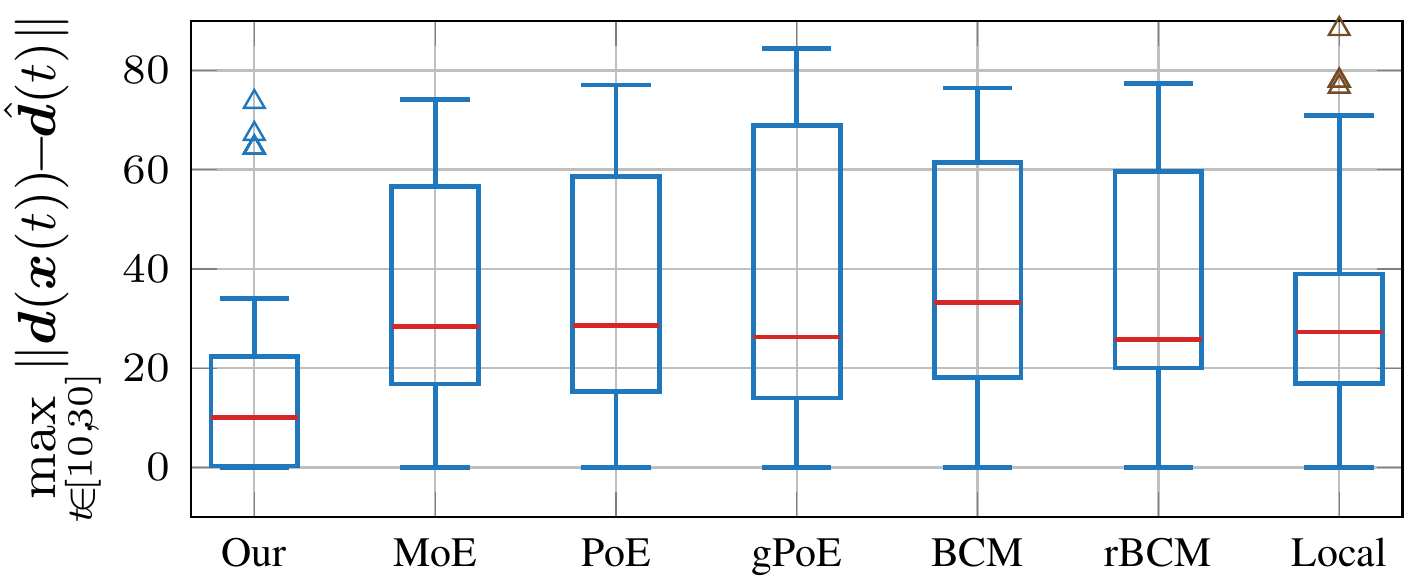}
		\caption{
			Maximal aggregated prediction error $\| \bm{d}(\bm{x}(t)) - \hat{\bm{d}}(t) \|$ in time domain $t \in [10, 30]$ for different cooperative learning-based control strategy, which is summarized from $100$ times Monte Carlo tests.
		}
		\label{figure_prediction_SteadyState}
	\end{figure}
	
	\begin{figure}[t] 
		\centering
		\includegraphics[width=0.48\textwidth]{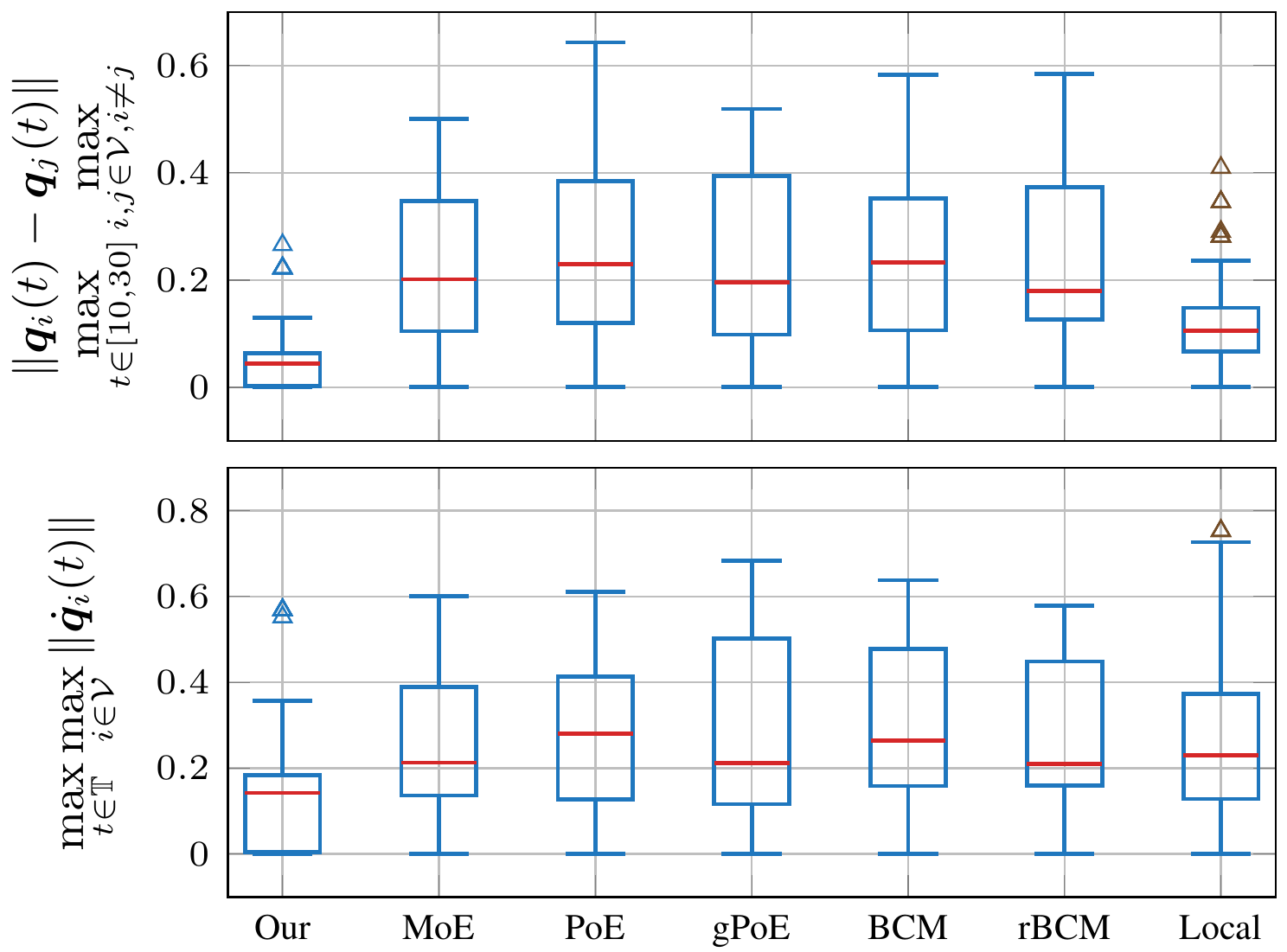}
		\caption{
			Maximal consensus error $\| \bm{q}_i(t) - \bm{q}_j(t) \|$ and velocity error $\| \dot{\bm{q}}_i(t) \|$ for all $i, j \in \mathcal{V}$ with $i \ne j$ in time domain $t \in [10, 30]$ for different cooperative learning-based control strategy, which is from $100$ times Monte Carlo tests.
		}
		\vspace{-0.3cm}
		\label{figure_error_SteadyState}
	\end{figure}
	
	\section{Conclusion}
	\label{section_conclusion}
	This paper presents a novel cooperative online learning-based control framework for multi-agent robotic systems with unknown dynamics, addressing the challenge of achieving second-order consensus under heterogeneous computational delays. 
	GP regression is employed at each individual agent, and an asynchronous cooperative learning strategy is developed to aggregate predictions despite differing computational delays and query point.
	Furthermore, a distributed control law is designed based on an adjoint MAS, which provides theoretical approximate consensus guarantees. 
	The effectiveness and superiority of the proposed methodology are demonstrated through numerical simulations on unmanned surface vehicles compared to existing cooperative approaches.
	
	\bibliographystyle{IEEEtran}
	\bibliography{ref}
	
	\def\BiographyDistance{-11pt}
	\begin{IEEEbiography}[{\includegraphics[width=1in,height=1.25in,clip,keepaspectratio]{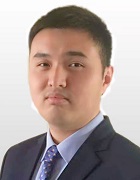}}]{Xiaobing Dai}
		received the B.Sc. mechanical engineering from the Tongji University, Shanghai, China, in 2018 with direction in mechatronics, building environment and civil engineering. He received double M.Sc degrees in Mechanical Engineering, Mechatronics and Robotics from the Technical University of Munich, Munich, Germany, in 2021. Since February 2022, he is a PhD student at the Chair of Information-oriented Control, TUM School of Computation, Information and Technology at the Technical University of Munich, Munich, Germany. His current research interests include efficient online machine learning, networked control systems, safe learning-based control.
	\end{IEEEbiography}
	\vspace{\BiographyDistance}
	
	\begin{IEEEbiography}[{\includegraphics[width=1in,height=1.25in,clip,keepaspectratio]{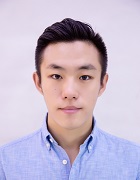}}]{Zewen Yang}
		(Member, IEEE)
		received the M.S. degree in control engineering from Northeast Forest University, in 2017. He pursued a Ph.D. in control science and engineering at the College of Intelligent Systems Science and Engineering, Harbin Engineering University, Harbin, China, from 2017 to 2019 and joined the Chair of Information-oriented Control, School of Computation, Information and Technology, the Technical University of Munich (TUM), Munich, Germany, in machine learning and data-driven control until 2023. 
		From 2023 to 2024, he was a postdoctoral researcher at the Robert Koch Institute, Berlin, Germany. 
		Since 2025, he has been with the Chair of Robotics and Systems Intelligence at the Munich Institute of Robotics and Machine Intelligence, TUM.
		His current research interests include multi-agent systems, cooperative learning, generative models, control theory, and general robotics.
	\end{IEEEbiography}
	\vspace{\BiographyDistance}
	
	\begin{IEEEbiography}[{\includegraphics[width=1in,height=1.25in,clip,keepaspectratio]{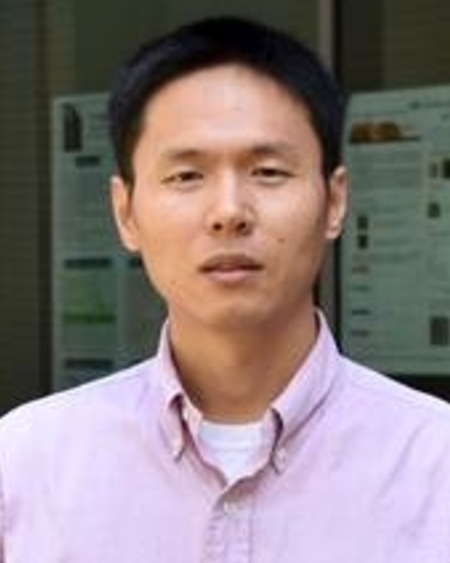}}]{Wei Ren}
		(Fellow, IEEE)
		received the Ph.D. degree in electrical engineering from Brigham Young University, Provo, UT, USA, in 2004.
		
		He is currently a Professor with the Department of Electrical and Computer Engineering, University of California at Riverside, Riverside, CA, USA. His research focuses on distributed control of multiagent systems.
		
		Dr. Ren was the recipient of the IEEE Control Systems Society Antonio Ruberti Young Researcher Prize in 2017 and the National Science Foundation CAREER Award in 2008. He was an IEEE Control Systems Society Distinguished Lecturer from 2020 to 2024.
	\end{IEEEbiography}
	\vspace{\BiographyDistance}
	
	\vfill
	
	\begin{IEEEbiography}[{\includegraphics[width=1in,height=1.25in,clip,keepaspectratio]{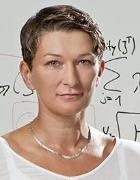}}]{Sandra Hirche}
		(Fellow, IEEE) 
		received the Dipl.-Ing degree in aeronautical engineering from the Technical University of Berlin, Berlin, Germany, in 2002, and the Dr. Ing. degree in electrical engineering from the Technical University of Munich, Munich, Germany, in 2005. From 2005 to 2007, she was awarded a Post-doctoral scholarship from the Japanese Society for the Promotion of Science at the Fujita Laboratory, Tokyo Institute of Technology, Tokyo, Japan. From 2008 to 2012, she was an Associate Professor with the Technical University of Munich. Since 2013, she has served as Technical University of Munich Liesel Beckmann Distinguished Professor and has been with the Chair of Information-Oriented Control, Department of Electrical and Computer Engineering, Technical University of Munich. She has authored or coauthored more than 150 papers in international journals, books, and refereed conferences. Her main research interests include cooperative, distributed, and networked control with applications in human--machine interaction, multirobot systems, and general robotics. 
		
		Dr. Hirche has served on the editorial boards of the IEEE Transactions on Control of Network Systems, the IEEE Transactions on Control Systems Technology, and the IEEE Transactions on Haptics. She has received multiple awards such as the Rohde \& Schwarz Award for her Ph.D. thesis, the IFAC World Congress Best Poster Award in 2005, and -- together with students -- the 2018 Outstanding Student Paper Award of the IEEE Conference on Decision and Control as well as Best Paper Awards from IEEE Worldhaptics and the IFAC Conference of Manoeuvring and Control of Marine Craft in 2009.
	\end{IEEEbiography}
	
	\vfill
\end{document}